\documentclass[12pt]{article}

\newcommand{\blind}{0}

\usepackage{amsthm,amsmath,amssymb,bbm,bm}
\usepackage{setspace, fancybox, fullpage}
\usepackage{comment}
\usepackage{multirow}
\usepackage{rotating}
\usepackage{subfigure}
\usepackage[colorlinks=true,allcolors=blue]{hyperref}%
\usepackage[round]{natbib} 
\usepackage{algorithm}
\usepackage[noend]{algpseudocode}

\usepackage{colortbl}

\newtheorem{theorem}{Theorem}

\newtheorem{assumption}{Assumption}

\newcommand{\bX}{\boldsymbol{X}}
\newcommand{\bx}{\boldsymbol{x}}

\newcommand{\Ind}[1]{\mathbbm{1}{\left\{ {#1} \right\} }}

\def\xv{\boldsymbol x}

\def\Xv{\boldsymbol X}

\newcommand{\betav}{\mbox{\boldmath{$\beta$}}}

\allowdisplaybreaks

\begin{document}

\def\spacingset#1{\renewcommand{\baselinestretch}%
{#1}\small\normalsize} \spacingset{1}

\if0\blind
{
  \title{\bf Learning Confidence Sets using Support Vector Machines}
  \author{Wenbo Wang and Xingye Qiao\thanks{Correspondence to: Xingye Qiao (e-mail: qiao@math.binghamton.edu). Wenbo Wang is a PhD student in the Department of Mathematical Sciences at Binghamton University, State University of New York, Binghamton, New York, 13902 (E-mail: wang2@math.binghamton.edu); and Xingye Qiao is Associate Professor in the Department of Mathematical Sciences at Binghamton University, State University of New York (E-mail: qiao@math.binghamton.edu).}
    }
  \maketitle
} \fi

\if1\blind
{
  \bigskip
  \bigskip
  \bigskip
  \begin{center}
    {\LARGE\bf Support Vector Machine with Confidence}
  \end{center}
  \medskip
} \fi

\bigskip

\begin{abstract}
The goal of confidence-set learning in the binary classification setting \citep{lei2014classification} is to construct two sets, each with a specific probability guarantee to cover a class. An observation outside the overlap of the two sets is deemed to be from one of the two classes, while the overlap is an ambiguity region which could belong to either class. Instead of plug-in approaches, we propose a support vector classifier to construct confidence sets in a flexible manner. Theoretically, we show that the proposed learner can control the non-coverage rates and minimize the ambiguity with high probability. Efficient algorithms are developed and numerical studies illustrate the effectiveness of the proposed method.
\end{abstract}

\noindent%
{\it Keywords:}  Support vector machine, Dual representation, Classification with confidence, Statistical learning theory
\vfill

\newpage
\spacingset{1.45} 

\setcounter{page}{1}
\abovedisplayskip=8pt
\belowdisplayskip=8pt

\section{Introduction}\label{sec:intro}

In binary classification problems, the training data consist of independent and identically distributed pairs $(\bm{X_i},Y_i)$, $i=1,2,...,n$ drawn from an unknown joint distribution $P$, with $\bm{X_i} \in \mathcal{X}\subset \mathbb{R}^p$, and $Y_i \in \{-1,1\}$. While the misclassification rate is a good assessment of the overall  classification performance, it does not directly provide confidence for the classification decision. \citet{lei2014classification} proposed a new framework for classifiers, named classification with confidence, using notions of \textit{confidence} and \textit{efficiency}. In particular, a classifier $\phi(\bm x)$ therein  is set-valued, i.e., the decision may be $\{-1\},\{1\}$, or $\{-1,1\}$. Such a classifier corresponds to two overlapped regions in the sample space $\mathcal{X}$, $C_{-1}$ and $C_{1}$, and they satisfy that $C_{-1} \cup C_{1} = \mathcal{X}$. With these regions, we have the set-valued classifier
$$\phi(\bm x)=\begin{cases}
\{-1\}, \makebox{when }\bm x\in C_{-1}\backslash C_{1}\\
\{1\}, \makebox{when }\bm x\in C_{1}\backslash C_{-1}\\
\{-1,1\}, \makebox{when }\bm x\in C_{-1} \cap C_{1}\\
\end{cases}$$
Those points in the first two sets are classified to a single class as by traditional classifiers. However, those in the overlap receive a decision of $\{-1,1\}$, hence may belong to either class. When the option of $\{-1,1\}$ is forbidden, the set-valued classifier degenerates to a traditional classifier.

\citet{lei2014classification} defined the notion of \textit{confidence} as the probability $100(1-\alpha_{j})\%$ that set $C_j$ covers population class $j$ for $j=\pm1$ (recalling the confidence interval in statistics). The notion of \textit{efficiency} is opposite to \textit{ambiguity}, which refers to the size (or probability measure) of the overlapped region named the \textit{ambiguity region}. In this framework, one would like to encourage classifiers to minimize the ambiguity when controlling the non-coverage rates. \citet{lei2014classification} showed that the best such classifier, the Bayes optimal rule, depends on the conditional class probability function $\eta(\bm{x}) = P(Y=1|\bm{X}=\bm{x})$. \citet{lei2014classification} then proposed to use the plug-in method, namely to first estimate $\eta(\xv)$ using, for instance, logistic regression, then plug the estimation into the Bayes solution. Needless to say, its empirical performance highly depends on the estimation accuracy of $\eta(\xv)$. However, it is well known that the latter can be more difficult than mere classification \citep{wang2007probability,furnkranz2010preference,wu2010robust}, especially when the dimension $p$ is large \citep{zhang2013multicategory}. 

Support vector machine \citep[SVM;][]{cortes1995support} is a popular classification method with excellent performance for many real applications. \citet{fernandez2014we} compared 179 classifiers on 121 real data sets and concluded that SVM was among the best and most powerful classifiers. To avoid estimating the conditional class probability $\eta(\xv)$, we propose a support vector classifier to construct confidence sets by empirical risk minimization. Our method is more flexible as it takes advantage of the powerful prediction power of support vector machine.

We show in theory that the population minimizer of our optimization is to some extent equivalent to the Bayes optimal rule in \cite{lei2014classification}. Moreover, in the finite-sample case, our classifier can control both non-coverage rates while minimizing the ambiguity.

A closely related problem is the Neyman-Pearson (NP) classification \citep{cannon2002learning,rigollet2011neyman} whose goal is to find a boundary for a specific null hypothesis class. It aims to minimize the probability that an observation from the alternative class falls into this region (the type II error) while controlling the type I error, i.e., the non-coverage rate for the null class. See \citet{tong2016survey} for a survey. Our problem can be understood as a two-sided NP classification problem. Other related areas of work are conformal learning, set-valued classification, or classification with reject and refine options. See \citep{shafer2008tutorial}, \cite{denis2016confidence}, \cite{tong2016survey}, \cite{vovk2017criteria}, \cite{herbei2006classification}, \cite{bartlett2008classification} and \cite{zhang2017reject}. 

The rest of the article is organized as follows. Some background information is provided in Section \ref{sec:back}. Our main method is introduced in Section \ref{sec:svmwc}.  A comprehensive theoretical study is conducted in Section \ref{sec:theorem}, including the Fisher consistency and novel statistical learning theory.  In Section \ref{sec:algorithm}, we present efficient algorithms to implement our method. The usefulness of our method is demonstrated using simulation and real data in Section \ref{sec:num}. Detailed proofs are in the Supplementary Material. 

\section{Background and notations}\label{sec:back}
We first formally define the problem and give some useful notations.

It is desirable to keep the ambiguity as small as possible. On the other hand, we would like as many class $j$ observations as possible to be covered by $C_j$. Consider predetermined non-coverage rates $\alpha_{-1}$ and $\alpha_{1}$ for the two classes. Let $P_{-1}$ and $P_{1}$ be the probability measure of $\bm{X}$ conditional on $Y = -1$ and $+1$. Conceptually, we formulate classification with confidence as the optimization below.
\begin{align}
\min_{C_{-1},C_{1}} &~ P\left(C_{-1}\cap C_{1} \right) \quad
\textrm{subject to}  ~ P_{j}(C_{j}) \geq 1-\alpha_{j}, \ j = \pm 1, \quad  C_{-1} \cup C_{1} = \mathcal{X}. \label{classificationwithconfidence} 
\end{align}
Here the constraint that $ P_{j}(C_{j}) \geq 1-\alpha_{j}$ means that $100(1-\alpha_{j})\%$ of the observations from class $j$ should be covered by region $C_j$.


\begin{figure}[h]\vspace{-1em}
	\centering
	\includegraphics[height = 5cm,width=0.4\textwidth]{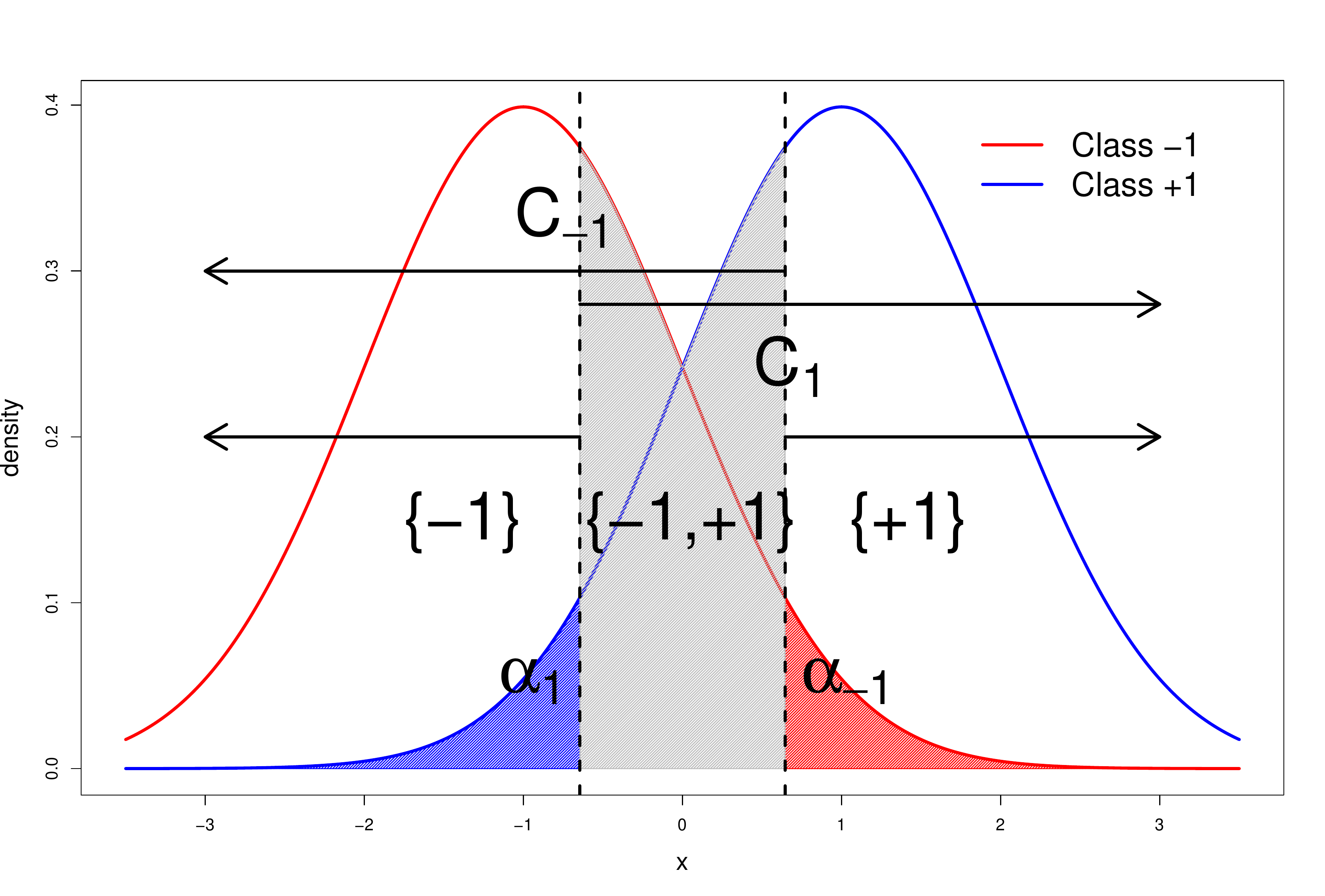}\vspace{-1em}
	\includegraphics[height = 5cm,width=0.4\textwidth]{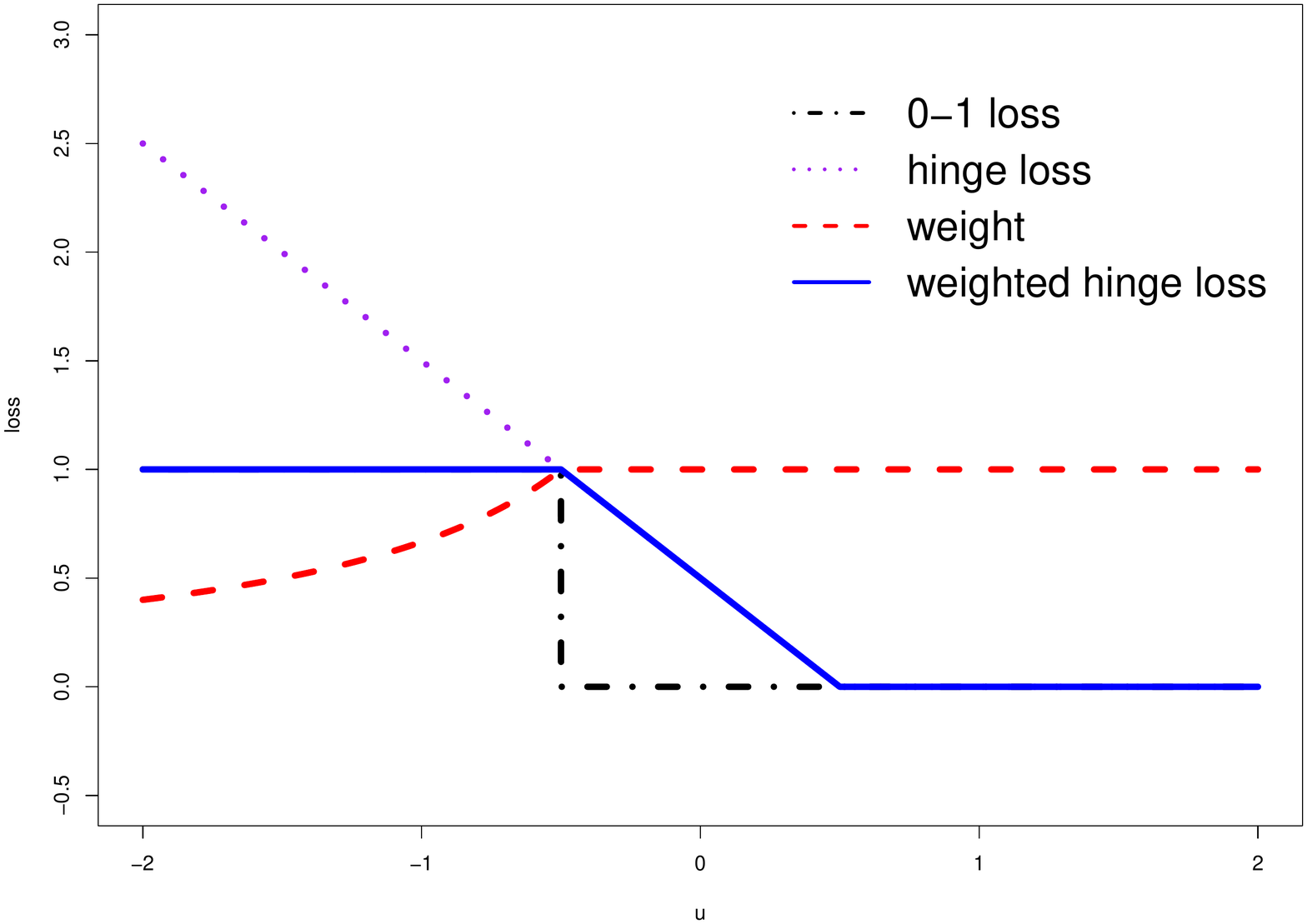}\vspace{-1em}
	\caption{The left panel shows the two definite regions and the ambiguity region in the case of symmetric Gaussian distributions. The right penal illustrates the weight function (see Section \ref{sec:svmwc}).}\label{illustration}
\end{figure}

Under certain conditions, the Bayes solution of this problem is: $C_{-1}^{*} = \{\bm x: \eta(\bm x) \le t_{-1} \} $ and $C_{1}^{*} = \{\bm x: \eta(\bm x) \ge t_{1}\}$ with $t_{-1}$ and $t_1$ satisfying that $P_{-1}(\eta(\bm X) \le t_{-1}) = 1-\alpha_{-1}$ and $P_{1}(\eta(\bm X) \ge t_{1}) = 1-\alpha_{1}$. A simple illustrative toy example with two Gaussian distributions on $\mathbb{R}$ is shown in Figure \ref{illustration}. The two boundaries are shown as the vertical lines, which lead to three decision regions, $\{-1\}$, $\{+1\}$, and $\{-1,+1\}$. The non-coverage rate $\alpha_{-1}$ for class $-1$ is shown on the right tail of the red curve (similarly, $\alpha_{1}$ for class $1$ on the left tail of blue curve.) In reality, the underlying distribution will be more complicated than a simple multivariate Gaussian distribution and the true boundary may be beyond linearity. In these cases, flexible approaches such as SVM will work better.

\section{Learning confidence sets using SVM}\label{sec:svmwc}
To avoid estimating $\eta$, we propose to solve the empirical counterpart of (\ref{classificationwithconfidence}) directly using SVM. Here, we present two variants of our method. We start with an original version to illustrate the basic idea. Then we introduce an improvement.

Unlike the regular SVM, the proposed classifier has two (not one) separating boundaries. They are defined as $\{\xv:f(\xv)=-\varepsilon\}$ and $\{\xv:f(\xv)=+\varepsilon\}$  where $f$ is the discriminant function, and $\varepsilon \geq 0$. The positive region $C_1$ is  $\{\xv:f(\xv)\ge -\varepsilon\}$ and the negative region $C_{-1}$ is $\{\xv:f(\xv)\le \varepsilon\}$. Hence when  $-\varepsilon\le f(\xv)\le \varepsilon$, observation $\xv$ falls into the ambiguity region $\{-1,1\}$.

Define $R(f,\varepsilon)=P(|Yf(\Xv)| \leq \varepsilon)$ the probability measure of the ambiguity. We may rewrite problem (\ref{classificationwithconfidence}) in terms of the function $f$ and threshold $\varepsilon$,
\begin{align}
\label{discriminatecwc}
\min_{\varepsilon \in \mathbb{R}^{+}, f} &~ R(f,\varepsilon),  \quad \textrm{subject to} \quad  P_{j}(Yf(\Xv) < -\varepsilon) \leq \alpha_{j}, \ j = \pm1. 
\end{align}
Replacing the probability measures above by the empirical measures, we can obtain,
\begin{align*}
\min_{\varepsilon \in \mathbb{R}^{+}, f} ~ \frac{1}{n} \sum_{i=1}^n\Ind{-\varepsilon\le f(\bm{x_i}) \le \varepsilon}, \quad
\textrm{subject to} ~   
\frac{1}{n_j}\sum_{i:y_i=j}\Ind{y_i f(\bm{x_i})\le -\varepsilon} \leq \alpha_j, ~ j = \pm 1.
\end{align*}
It is easy to show that as long as the equalities in the constraints are achieved at the optimum, we can obtain the same minimizer if the objective function is changed to ${\frac{1}{n}\sum_{i=1}^n\Ind{y_{i}f(\bm{x_i})-\varepsilon \leq 0}}$.

For efficient and realistic optimization, we replace the indicator function $\Ind{u\le 0}$ in the objective function and constraints by the Hinge loss function $(1-u)_+$. The practice of using a surrogate loss to bound the non-coverage rates has been widely used in the literature of NP classification, see \cite{rigollet2011neyman}. To simplify the presentation, we denote $H_{a}(u) = (1+a-u)_+$ as the $a$-Hinge Loss and it can be seen that $H_{a}(x)$ coincides with the original Hinge loss when $a = 0$. Our initial classifier can be represented by the following optimization:
\begin{align}
\min\limits_{\varepsilon \in \mathbb{R}^{+}, f} ~ \frac{1}{n}\sum\limits_{i=1}^{n}H_{\varepsilon}(y_{i}f(\bm{x_i}))+\lambda J(f), \quad
\text{subject to}  ~  \frac{1}{n_{j}}\sum\limits_{i:y_{i}=j}H_{-\varepsilon}(y_{i}f(\bm{x_i})) \leq \alpha_{j},~j=\pm 1  \label{svmwcoriginal} 
\end{align}
Here $J$ is a regularization term to control the complexity of the discriminant function $f$. When $f$ takes the linear form of  $f(\bm{x})=\bm{x}^{T}\betav+b$, $J(f)$ can be $L_{2}$-norm $\|\betav\|^2$ or $L_{1}$-norm $|\betav|$.

In SVM, $yf(\xv)$ is called the functional margin, which measures the signed distance from $\xv$ to the boundary $\{\xv:f(\xv)=0\}$. Positive and large value of $yf(\xv)$ means the observation is correctly classified, and is far away from the boundary. In our situation, we compare $yf(\xv)$ with $+\varepsilon$ and $-\varepsilon$ respectively. If $yf(\xv) < -\varepsilon$, then $\xv$ is not covered by $C_{y}$ (hence is misclassified, in the classification language). On the other hand, if $yf(\xv)\le \varepsilon$, then $\xv$ either satisfies that $yf(\xv) < -\varepsilon$ as above, or falls into the ambiguity, which is why we try to minimize the sum of $H_{\varepsilon}(y_{i}f(\bm{x_i}))$.

By constraining $\sum_{y_{i}=j}H_{-\varepsilon}(y_{i}f(\bm{x_i}))$ for both classes, we aim to control the non-coverage rates. Since $H_{-\varepsilon}(u) \geq \Ind{u<-\varepsilon}$ (the latter indicates the occurrence of non-coverage) for negatively large  $u$. It may be more conservative by using the Hinge loss than the indicator function $\Ind{y_{i}f(\bm{x_i})<-\varepsilon}$ in the constraint to control the non-coverage rates. We alleviate this problem by imposing a weight $w_{i}$ to each observation in the constraint. In particular, this weight is chosen to be $w_i=\max\{1,H_{-\varepsilon}(y\hat f(\xv))\}^{-1}$, where $\hat f$ is a reasonable guess of the final minimizer $f$. Our goal is to weight the Hinge loss in the constraint, $w_{i}H_{-\varepsilon}(y_{i}f(\bm{x_i}))$, so that it approximates the indicator function $\Ind{y_{i}f(\bm{x_i})<-\varepsilon}$. This may be illustrated by Figure \ref{illustration} in which the blue bold line is the result of multiplying the weight (red dashed) by the Hinge loss (purple dotted), which is close to the indicator function (black dot-dashed). Note that by weighting the Hinge loss, the impact of those observations with very negatively large $u=yf(\xv)$ value is reduced to 1. The adaptive weighted version of our method changes constraint (\ref{svmwcoriginal}) to $\frac{1}{n_{j}}\sum_{i:y_{i}=j}w_{i}H_{-\varepsilon}(y_{i}f(\bm{x_i})) \leq \alpha_{j}, j=\pm 1$.

In practice, we adopt an iterative approach, and use the estimated $f$ from the previous iteration to calculate the weight for each observation at the current iteration. We start with equal weights for each observation, solve the optimization problem with the weights obtained in the last iteration, and then calculate the new weights for the next iteration. \citet{wu2013adaptively} first used this idea in their work of adaptively weighted large margin classifiers for the purpose of robust classification.

\section{Theoretical Properties}\label{sec:theorem}
In this section we study the theoretical properties of the proposed method. We start with population level properties in Section \ref{subsec:fisher}. In Section \ref{subsec:finite}, we discuss the finite-sample properties using novel statistical learning theory.
\subsection{Fisher consistency and excess risk}\label{subsec:fisher}

Assume that $P_{-1}$ and $P_{1}$ are continuous with density function $p_{-1}$ and $p_{1}$, and $\pi_{j}=P(Y=j)$ is positive for $j=\pm 1$. Moreover, $\eta(\bm X)$ is continuous, and $t_{-1}$ and $t_{1}$ are quantiles of $\eta(\bm X)$. They satisfies $P_{-1}(\eta(\bm X) \leq t_{-1}) = 1-\alpha_{-1}$ and $P_{1}(\eta(\bm X) \geq t_{1}) = 1-\alpha_{1}$. We need to make assumptions on the difficulty level of the classification task. In particular, the classification should be difficult enough so that overlapping regions is meaningful (otherwise, there will be almost no ambiguity even at small non-coverage rates.)
\begin{assumption}\label{assumption1}
	$t_{-1} \geq \frac{1}{2} \geq t_{1}$.
\end{assumption}
\begin{assumption}\label{assumption2} 
	$\exists c>0$,  $t_{-1}-c \geq \frac{1}{2} \geq t_{1}+c$.
\end{assumption}
Each assumption implies that the union of $C_{-1}^{*} = \{\bm x: \eta(\bm x) \le t_{-1}\}$ and $C_{1}^{*} = \{\bm x: \eta(\bm x) \ge t_{1}\}$ is $\mathcal{X}$. Otherwise, there will be a gap around the boundary $\{\bm x:\eta(\bm x)=1/2\}.$ It is easy to see that Assumption \ref{assumption2} is stronger than Assumption \ref{assumption1}.

Fisher consistency concerns the Bayes optimal rule, which is the minimizer of problem (\ref{discriminatecwc}). In (\ref{hingedobjective}) below, we replace the loss function in the objective function of (\ref{discriminatecwc}) with risk under the Hinge loss.
\begin{align}\label{hingedobjective}
\min  &~  R_{H}(f,\varepsilon) ,  \quad  \textrm{subject to} ~ P_{j}(Yf(X) < -\varepsilon) \leq \alpha_{j}, \ j = \pm 1, 
\end{align}
where $R_{H}(f,\varepsilon) = E[H_{\varepsilon}(Yf(X))]$.

Theorem \ref{equivalenttrueloss} shows that for any fixed $\varepsilon$, the minimizer of (\ref{hingedobjective}) is the same as the Bayes rule \cite{lei2014classification}.
\begin{theorem}\label{equivalenttrueloss}
	Under Assumption \ref{assumption1}, for any fixed $\varepsilon \geq 0$, function 
	\begin{align*}
	f^{*}(x) =  \left\{\begin{array}{ll}
	1+\varepsilon, & \eta(x) > t_{-} \\
	\varepsilon \cdot\makebox{sign}(\eta(x)-\frac{1}{2}), & t_{+} \leq \eta(x) \leq t_{-} \\ 
	-(1+\varepsilon), &  f(x) < t_{+}
	\end{array}\right.
	\end{align*}
	is the minimizer to both (\ref{discriminatecwc}) and (\ref{hingedobjective}). 
\end{theorem}
A key result in \citet{bartlett2006convexity} was that the excess risk of 0-1 classification loss is bounded by the excess risk of surrogate loss. Here we show a similar result for the confidence set problem. That is, the excess ambiguity $R(f,\varepsilon) - R(f^{*},\varepsilon)$  vanishes as $R_{H}(f,\varepsilon) - R_{H}(f^{*},\varepsilon)$ goes to 0. 
\begin{theorem}\label{excessriskbound}
	Under Assumption (2), for any $\varepsilon\ge 0$, and $\forall f$ satisfying the constraints in (\ref{discriminatecwc}), there exists $C^{'} = \frac{1}{4c^{2}}+\frac{1}{2c} > 0$ such that the following inequality holds,
	\begin{align*}
	C^{'} (R_{H}(f,\varepsilon)-R_{H}(f^{*},\varepsilon)) \geq R(f) - R(f^{*}).
	\end{align*}
\end{theorem}
Note that $C^{'}$ does not depend on $\varepsilon$.

\subsection{Finite-sample properties}\label{subsec:finite}
Denote the Reproducing Kernel Hilbert Space (RKHS)  with bounded norm as $\mathcal{H}_{K}(s) = \{f:\mathcal{X}\rightarrow \mathbb{R}|f(\bx) = h(\bx)+b, h\in \mathcal{H}_{K},||h||_{\mathcal{H}_{K}}\leq s ,b \in \mathbb{R}\}$ and $r = \sup_{x \in \mathcal{X}}K(x,x)$. For a fixed $\varepsilon$, define the space of constrained discriminant functions as $\mathcal{F}_{\varepsilon}((\alpha_{-1},\alpha_{1})) = \{f:\mathcal{X}\rightarrow \mathbb{R}|E(H_{-\varepsilon}(Yf(\bX))|Y=j) \leq \alpha_{j},j = \pm 1\}$, and its empirical counterpart as $\hat{\mathcal{F}}_{\varepsilon}((\alpha_{-},\alpha_{+})) = \{f:\mathcal{X}\rightarrow \mathbb{R}|n_j^{-1}\sum_{i:y_{i}=j}H_{-\varepsilon}(y_{i}f(\bx_{i})) \leq \alpha_{j}, j = \pm 1 \}$. Moreover, we define the feasible function space $\mathcal{F}_{\varepsilon}(\kappa,s) = \mathcal{H}_{K}(s) \cap \mathcal{F}_{\varepsilon}((\alpha_{-1}-\frac{\kappa}{\sqrt{n_{-1}}},\alpha_{1}-\frac{\kappa}{\sqrt{n_{1}}}))$ and its empirical counterpart $\hat{\mathcal{F}}_{\varepsilon}(\kappa,s) = \mathcal{H}_{K}(s) \cap \hat{\mathcal{F}}_{\varepsilon}((\alpha_{-1}-\frac{\kappa}{\sqrt{n_{-1}}},\alpha_{1}-\frac{\kappa}{\sqrt{n_{1}}}))$. Lastly, consider a subset of the Cartesian product of the above feasible function space and the space for $\varepsilon$, $\mathcal{F}(\kappa,s) = \{(f,\varepsilon),f \in \mathcal{F}_{\varepsilon}(\kappa,s),\varepsilon \ge 0\}$ and its empirical counterpart $\hat{\mathcal{F}}(\kappa,s) = \{(f,\varepsilon),f \in \hat{\mathcal{F}}_{\varepsilon}(\kappa,s),\varepsilon \ge 0\}$. Then optimization problem (\ref{svmwcoriginal}) of our proposed method can be written as 
\begin{align}\label{optimizationformula}
\min\limits_{(f,\varepsilon) \in \hat{\mathcal{F}}(0,s)} & ~{\sum_{i=1}^{n}{H_{\varepsilon}(y_{i}f(\bx_{i}))}}
\end{align}
In Theorem \ref{controlerrorrate}, we give the finite-sample upper bound for the non-coverage rate.
\begin{theorem}\label{controlerrorrate}
	Let $(f,\varepsilon)$ be a solution to optimization problem (\ref{optimizationformula}), then with probability at least $1-2\zeta$, $Z = \sqrt{sr}/\sqrt{n}$, $T_{n}(\zeta) = \{2srlog(1/\zeta)/n\}^{1/2}$ and $r = \sup_{\mathcal{X}}{K(x,x)}$
	\begin{align*}
	P_{j}(Yf(X) < -\varepsilon) \leq \frac{1}{n_j}E[H_{-\varepsilon}(Yf(\bm{X})) | Y = j] 
	\leq \sum_{y_{i}=j}{H_{-\varepsilon}(y_{i}f(\bx_{i}))}+3T_{n_{j}}(\zeta)+Z(n_{j})
	\end{align*}
\end{theorem} 
Theorem \ref{controlerrorrate} suggests that if we want to control the non-coverage rate on average at the nominal $\alpha_{-1}$ or $\alpha_{1}$ rates with high probability, we should choose the $\alpha_{-1}$ or $\alpha_{1}$ values to be slightly smaller than the desired ones in optimization (\ref{svmwcoriginal}) in practice. In particular, we need to make $\sum_{y_{i}=j}{H_{-\varepsilon}(y_{i}f(\bx_{i}))}+3T_{n_{j}}(\zeta)+Z(n_{j}) \leq \alpha_{j}$. Note that the remainder terms $3T_{n_{j}}(\zeta)+Z(n_{j})$ will vanish as $n_{-1},n_{1}\rightarrow\infty$.

The next theorem ensures that the empirical ambiguity probability from solving (\ref{optimizationformula}) based on a finite sample will converge to the ambiguity given by the solution on an infinite sample (under the constraints  $E(H_{-\varepsilon}(Yf(\bX))|Y=j) \leq \alpha_{j},j = \pm 1$).


\begin{theorem}\label{finiteerrorbound}
	Let $(\hat{f},\hat{\varepsilon})$ be the solution of the optimization problem (\ref{svmwcstronger})
	\begin{align}
	\min_{(f,\varepsilon)\in\hat{\mathcal{F}}(\kappa,s)}{\sum_{i=1}^{n}{H_{\varepsilon}(y_{i}f(\bx_{i}))}} \label{svmwcstronger}
	\end{align}
	with $\kappa = (6log(\frac{1}{\zeta})+1)\sqrt{sr}$. Then with probability $1-6\zeta$, and large enough $n_{-1}$ and $n_{1}$ we have \\ 
	(i). $\hat{f} \in \mathcal{F}_{\hat{\varepsilon}}(0,s)$, and \\
	(ii). $R_{H}(\hat{f},\hat{\varepsilon})-\min\limits_{(f,\varepsilon) \in \mathcal{F}(0,s)}{R_{H}(f,\hat{\varepsilon})} \leq \kappa({2}{{n}^{-1/2}}+4\min{\{\nu_{-1},\nu_{1}\}}^{-1}\min{\{\sqrt{n_{-1}},\sqrt{n_{1}}\}}^{-1})$.
\end{theorem}

\section{Algorithms}\label{sec:algorithm}
In this section, we give details of the algorithm. Similar to the SVM implementation, we propose to solve the dual problem. We start with the linear SVM with $L_2$ norm for illustrative purposes. After introducing two sets of slack variables, $\eta_{i}=(1-\varepsilon-y_i(\bm{x_i}^{T}\betav+b))_{+}$ and $\xi_{i}=(1+\varepsilon-y_i(\bm{x_i}^{T}\betav+b))_{+}$, we can show that
(\ref{svmwcoriginal}) is equivalent to (\ref{svmversionofsvmwc}),
\begin{align}
\min_{\Theta}  ~&  {\frac{1}{2}||\betav||_{2}^{2} + \lambda' \sum\limits_{i}^{n}{\xi_{i}}} \label{svmversionofsvmwc}\\ 
\nonumber \textrm{subject to} ~   y_i(\bm{x_i}^{T}\betav+b) \geq 1+\varepsilon-\xi_i, \quad & y_i(\bm{x_i}^{T}\betav+b) \geq 1-\varepsilon-\eta_i \quad \textrm{for all} \ i=1,2,...,n, \\
\nonumber  \xi_i \geq 0, \quad \sum\limits_{y_{i}=-1}{w_{i}\eta_{i}} \leq n_{-1}\alpha_{-1}, \quad & \eta_i \geq 0, \quad \sum\limits_{y_{i}=1}{w_{i}\eta_{i}} \leq n_{1}\alpha_{1}, \quad \varepsilon \geq 0. 
\end{align}
Here $\Theta$ is the collection of all variables of interest, namely $\Theta = \{\varepsilon,\betav,b,\{\xi_i\}_{i=1}^{n},\{\eta_i\}_{i=1}^{n}\}$. We can then solve it via the quadratic programming below,
\begin{align}
\min_{\Theta'}~&{\frac{1}{2}\sum_{i=1}^{n}\sum_{j=1}^{n}}{{(\zeta_{i}+\tau_{i})(\zeta_{j}+\tau_{j})y_{i}y_{j}} \xv_{i}'\xv_{j}}  
-\sum_{i=1}^{n}{\zeta_i}-\sum_{i=1}^{n}{\tau_i}+ n_{-1}\alpha_{-1}\theta_{-1} + n_{1}\alpha_{1}\theta_{1} \label{dual}\\
\nonumber \textrm{subject to} ~ & 0 \leq \zeta_{i} \leq \lambda', \quad  0 \leq \tau_{i} \leq \theta_{y_{i}}w_{i},   \quad 
\sum_{i=1}^{n}{\zeta_{i}y_{i}} + \sum_{i=1}^{n}{\tau_{i}y_{i}}  = 0,\quad \sum_{i=1}^{n}{\zeta_{i}}  -\sum_{i=1}^{n}{\tau_{i}}  \geq 0.
\end{align} 
Here $\Theta'=\{\{\zeta_i\}_{i=1}^{n},\{\tau_i\}_{i=1}^{n},\theta_{-1},\theta_{1}\}$ consists of all the variables in the dual problem. The above optimization may be solved by any efficient quadratic programming routine. After solving the dual problem, we can find $\betav$ by $\betav = \sum_{i}^{n}{\zeta_{i}y_{i}\bm{x_i}}+\sum_{i}^{n}{\tau_{i}y_{i}\bm{x_i}}$. Then we can plug $\betav$ into the primal problem and find $b$ and $\varepsilon$ by linear programming.

For nonlinear $f$, we can adopt the widely used `kernel trick'. Assume $f$ belongs to a Reproducing Kernel Hilbert Space (RKHS) with a positive definite kernel $K$, $f(x) = \sum_{i=1}^{n}c_{i}{K(\xv_{i},\xv)} + b$. In this case the dual problem is the same as above except that $\xv_{i}'\xv_{j}$ is replaced by $K(\xv_{i},\xv_{j})$. After the solution has been found, we then have $c_{i} = \zeta_{i} + \tau_{i}$. Common choices for the kernel function includes the Gaussian kernel and the polynomial kernel.

\section{Numerical Studies}\label{sec:num}
In this section, we compare our confidence-support vector machine (CSVM) method and methods based on the plug-in principal, including $L_2$ penalized logistic regression \citep{le1992ridge}, kernel logistic regression \citep{zhu2005kernel}, kNN \cite{altman1992introduction}, random forest \citep{liaw2002classification} and SVM \citep{cortes1995support,platt1999probabilistic} using both simulated and real data.
\subsection{Simulation}\label{subsec:simulation}
We study the numerical performance over a large variety of sample sizes. In each case, an independent tuning set with the same sample size as the training set is generated for parameter tuning. The testing set has 20000 observations (10000 or nearly 10000 for each class). We run the simulation 100 times and report the average and standard error. Both non-coverage rates are set to 0.05. 

We select the best parameter $\lambda$ and the hyper-parameter for kernel methods as follows. We search for the optimal $\rho$ in the Gaussian kernel  $\exp{(-\|x-y\|^{2}/\rho^2)}$ from the grid $10^{\{-0.5,-0.25,~0,~0.25,~0.5,~0.75,~1\}}$ and the optimal degree for polynomial kernel from $\{2,~3,~4\}$. For each fixed candidate hyper-parameter, we choose $\lambda$ from a grid of candidate values ranging from $10^{-8}$ to $10^{4}$ by the following two-step searching scheme. We first do a rough search with a larger stride $\{10^{-8}, 10^{-7.5}, \dots, 10^{4}\}$ and get the best parameter $\lambda_1$. Then we do a fine search from $\lambda_1 \times \{10^{-0.5},10^{-.4},\dots,10^{0.5}\}$.  After that, we choose the optimal pair which gives the smallest tuning ambiguity and has the two non-coverage rates for the tuning set controlled.

To improve the performance, we make use of the suggested robust implementation in \cite{lei2014classification} for all the methods. Following \cite{lei2014classification}, we first obtain an estimate of $\eta$ or a monotone proxy of it such as the discriminant function $f$ in SVM, then choose thresholds $\hat{t}_{-1}$ and $\hat{t}_{1}$ which are two sample quantiles of $\widehat\eta(\bm x)$ (or $f(\bm x)$) among the tuning set so that the non-coverage rates for the tuning set match the nominal rates. The final predicted sets are induced by thresholding $\widehat\eta(\bm x)$ (or $f(\bm x)$) using $\hat{t}_{-1}$ and $\hat{t}_{1}$. 

Because there are two non-coverage rates and one ambiguity size to compare here, how to make fair comparison becomes a tricky problem since one classifier can sacrifice the non-coverage rate to gain in ambiguity. One by-product of the robust implementation above is that the non-coverage rate of most of the methods will become very similar and we only need to compare the size of the ambiguity. 

We also include a simple SVM approach whose discriminant function is obtained in the traditional way, but which induces confidence sets by thresholding in the same way described above.

We consider three different simulation scenarios. In the first scenario we compare the linear approaches (SVM and penalized logistic regression), while in the next two cases we consider nonlinear methods. In all cases, we add additional noise dimensions to the data. These noise covariates are normally distributed with mean $\bm{0}$ and $\Sigma = \text{diag}(1/p)$, where $p$ is the total dimension of the data.

\textbf{Example 1 (Linear model with nonlinear Bayes rule):} In this scenario, we have two normally distributed classes with different covariance matrices. In particular, denote $X|Y=j \sim \mathcal{N}(\mu_{j},\Sigma_{j})$ for $j=\pm 1$, then $\mu_{-1} = (-2,1)^{T}$, $\mu_{1} = (1,0)^{T}$, and $\Sigma_{-1} = \text{diag}(2,\frac{1}{2})$, $\Sigma_{1} = \text{diag}(\frac{1}{2},2)$. The prior probabilities of both classes are the same. Lastly, we add eight dimensions of noise covariates to the data. The data are illustrated in the left penal of Figure \ref{fig:simulation}. 
We compare linear CSVM, and the plug-in methods $L_2$ penalized logistic regression \citep{friedman2010regularization} and simple linear SVM to estimate $\eta$.

\textbf{Example 2 (Moderate dimensional polynomial boundary):} This case is similar to the one in \citep{zhang2008variable}. First we generate $x_1 \sim \makebox{Unif}[-1,1]$ and $x_2 \sim \makebox{Unif}[-1,1]$. Define functions $f_{j}(\xv) = j(-3.6 x_1^2 + 7.2 x_2^2 - 0.8),j = \pm 1$. Then we set $\eta(\xv) = f_{1}(\xv)/(f_{-1}(\xv)+f_{1}(\xv))$, where $\xv = (x_1,x_2)$. We then add 98 covariates on top of the 2-dimensional signal. The data are illustrated in the middle penal of Figure \ref{fig:simulation}. In this scenario, we choose to use the polynomial kernel for all the kernel based methods. 

\textbf{Example 3 (High-dimensional donut):} We first generate a two-dimensional data, $(r_i,\theta_i)$ where $\theta_i\sim \makebox{Unif}[0,2\pi]$, $r_i|(Y=-1)\sim \makebox{Uniform}[0,1.2]$, and $r_i|(Y=+1)\sim \makebox{Unif}[0.8,2]$. Then we define the two-dimensional $\bm{X}_i=(r_i \cos(\theta_i), r_i\sin(\theta_i))$. The data are illustrated in the right penal of Figure \ref{fig:simulation}.  We then add 498 covariates on top of the 2-dimensional signal.  We use the Gaussian kernel, $K(x,y;\rho) = \exp{(-\|x-y\|^{2}/\rho^2)}$ for all the kernel based methods.

\begin{figure}[h]
	\centering
	\vspace{-0.3cm}
	\includegraphics[width=0.32\textwidth,trim={0 2em 0 1em}]{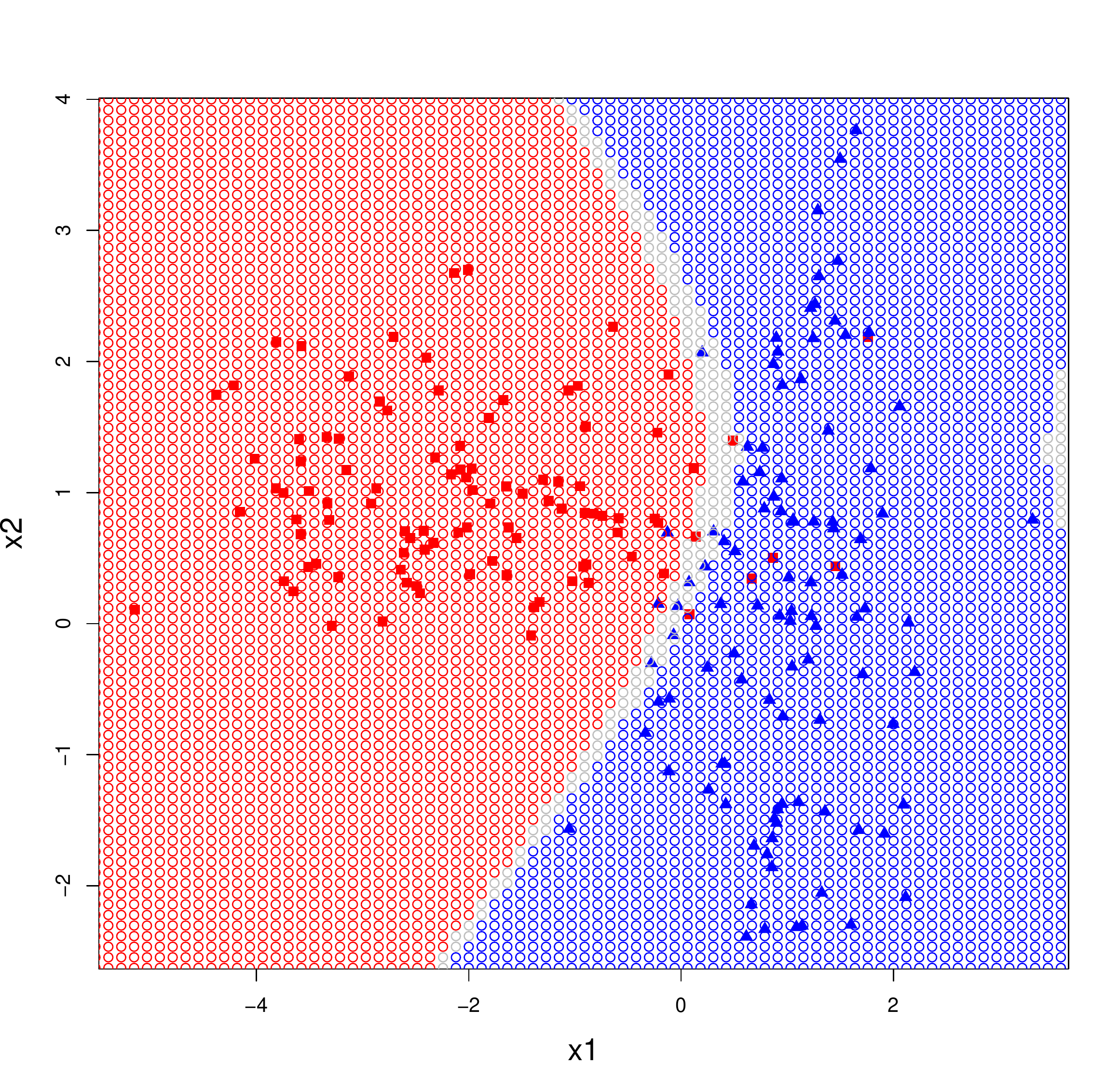}
	\includegraphics[width=0.32\textwidth,trim={0 2em 0 1em}]{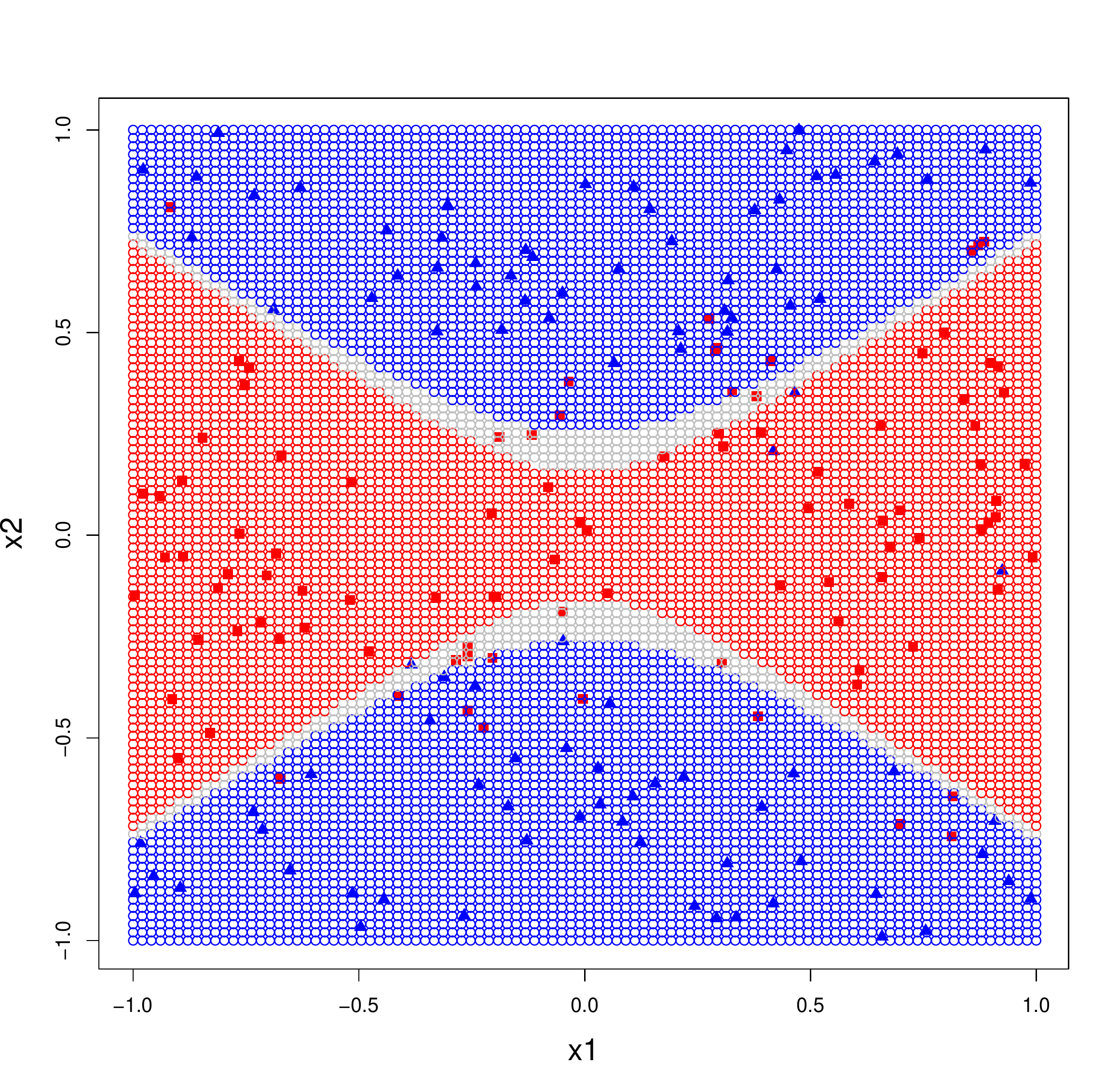}
	\includegraphics[width=0.32\textwidth,trim={0 2em 0 1em}]{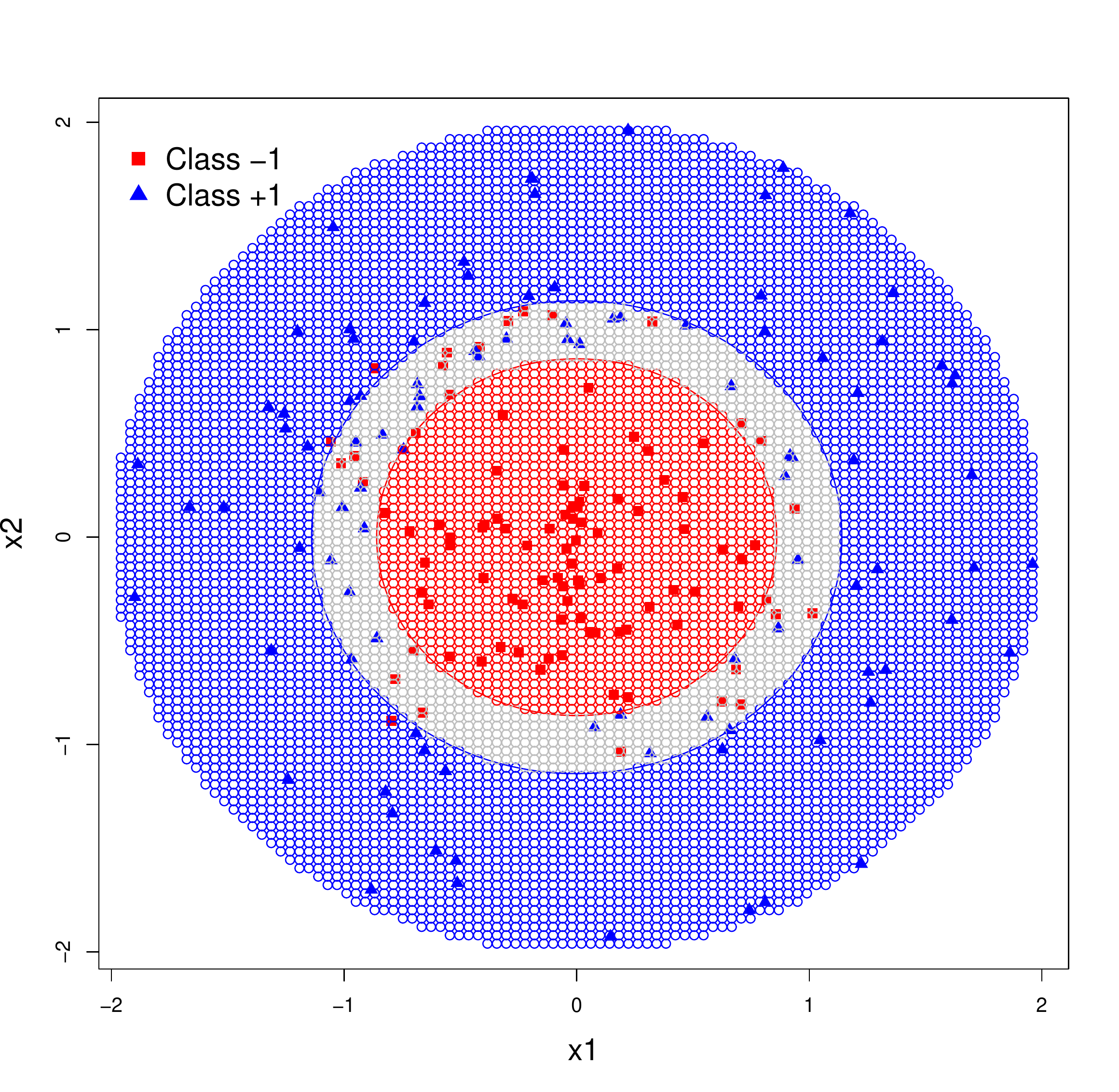}
	\vspace{-1em}
	\caption{Scatter plots of the first two dimensions for the simulated data with Bayes rules showing the two definite regions and the ambiguity region.}\label{fig:simulation}
\end{figure}

All methods are improved using the robust implementation. The results are reported in Figure \ref{fig:simulation_outcome}. We also show the performance of CSVM with weighting but without robust implementation. For {Example 1}, our CSVM method gives a significantly smaller ambiguity than either logistic regression or naive SVM. In {Example 2} and {Example 3}, our method gives a smaller or at least comparable ambiguity to the best plug-in method, which is kernel logistic regression. Our weighted CSVM performs the best when sample size is small in the linear case and it outperforms kNN, Random Forest and naive SVM in nonlinear cases. The naive SVM method which directly uses simple SVM to conduct confidence set learning performs significantly worse than all the other methods in nonlinear cases. The non-coverage rates (not shown here) of CSVM, random forest, kernel logistic regression and naive SVM methods are close to each other while CSVM without robust implmentation and kNN have similar non-coverage rates. A detailed comparison can be found in the Supplementary Material.

\begin{figure}[H]
	\begin{center}
		\includegraphics[width=0.9\textwidth,trim={0 2em 0 3em}]{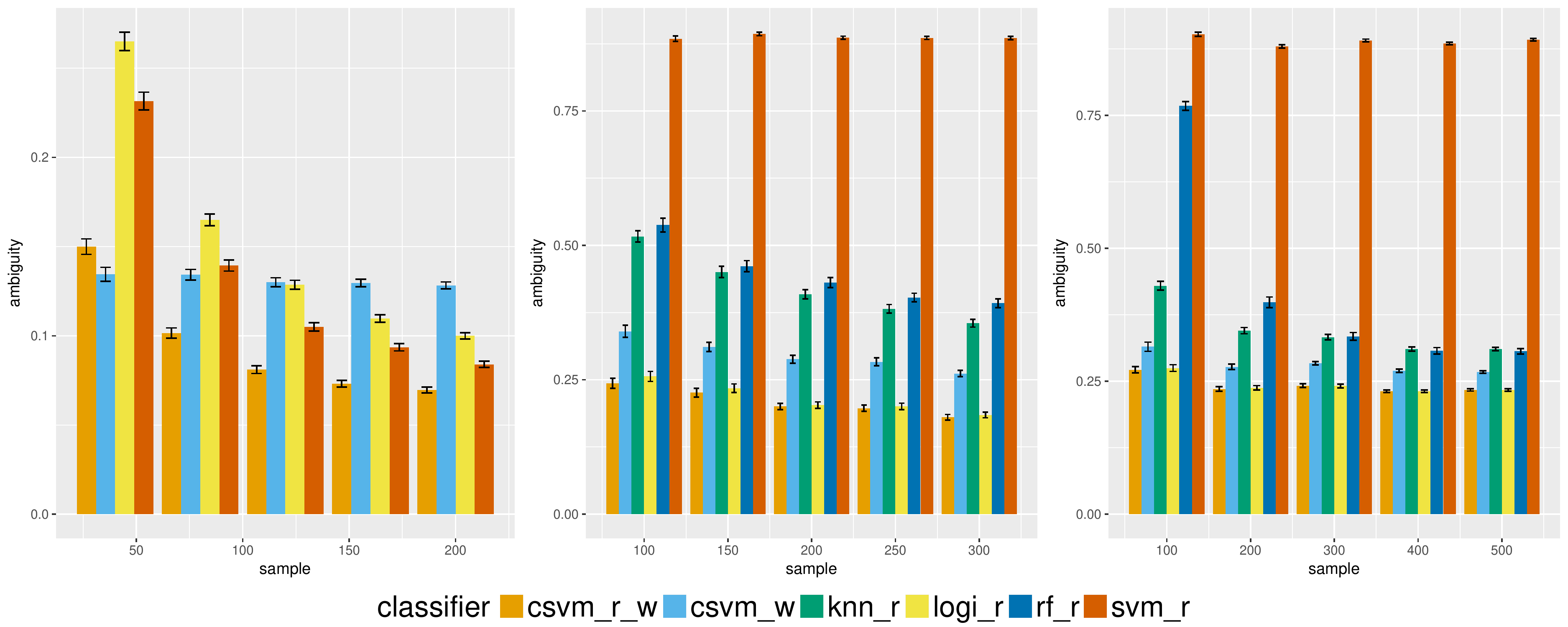}
	\end{center}
	\caption{Outcome of ambiguities in three simulation settings. Non-coverage rates are similar among different methods and are not shown here. CSVM has the smallest ambiguity.}\label{fig:simulation_outcome}
\end{figure}

\subsection{Real Data Analysis}\label{subsec:real}
We conduct the comparison on the hand-written zip code data \cite{lecun1989backpropagation}. The data set consists of many $16 \times 16$ pixel images of handwritten digits. It is widely used in the classification literature. There are both training and testing sets defined in it. \citet{lei2014classification} used the same dataset for illustrating the plug-in methods. We choose this dataset to directly compare with the plug-in methods.

Following \citet{lei2014classification}, to form a binary classification problem, we use the subset of the data containing digits $\{0,6,8,9\}$. Images with digits 0, 6, 9 are labeled as class $-1$ (they are digits with one circle) and those with digit 8 (two circles) are labeled as class $+1$. Previous studies \cite{shafer2008tutorial} pointed out that there was discrepancies between the training and testing set of this data set. So in this study we first mixed the training and testing data and then randomly split into new training, tuning and testing data. The training and tuning data both have sample size 800, with 600 from class $-1$ and 200 from class $1$ to preserve the unbalance nature of the data set. During training, we oversample class $1$ by counting each observation three times to alleviate the unbalanced classes issue.

Although \citet{lei2014classification} set both nominal non-coverage rates to be 0.05 in their study which focused on linear methods, it needs to be pointed out that many nonlinear classifiers, such as SVM with Gaussian kernel, can achieve this non-coverage rate without introducing any ambiguity. Therefore we reduce the non-coverage rate to 0.01 for both classes to make the task more challenging.

We apply Gaussian kernel for CSVM, and compare with kernel logistic regression with Gaussian kernel, random forest, kNN and naive SVM with Gaussian kernel on this data set. 

\begin{table}[H]\small
	\centering
	\begin{tabular}{lllllll}
		\hline
		Classifier & CSVM       & CSVM(r)   & KNN(r)         & KLR(r)      & RF(r)         & SVM-Prob(r)   \\
		\hline
		Non-coverage(-1)    & 0.05(0.005) & {1.02}(0.05) & 0.81(0.04)  & 0.98(0.05) & 0.95(0.04) & {1.00}(0.05) \\
		Non-coverage(+1)    & 0.56(0.06)  & {1.19}(0.11) & 1.04(0.09)  & 1.25(0.10) & 1.10(0.11) & {1.25}(0.11) \\
		Ambiguity  & 8.29(0.18)  & {2.52}(0.13) & 10.21(2.12) & 3.46(0.17) & 7.55(0.37) & {2.60}(0.13) \\
		\hline
	\end{tabular}
	\vspace{0.1cm}
	\caption{CSVM gives better or comparable outcome to the best plug-in method.}\label{real_data_outcome}\vspace{-2em}
\end{table}

The results are summarized in Table \ref{real_data_outcome} with numbers in percentage. CSVM gives better results than all the plug-in methods. We plot the zip code data using t-distributed stochastic neighbor embedding (t-SNE) \citep{maaten2008visualizing} to give a visualization of our method and the data.

\begin{figure}[H]
	\begin{center}
		\includegraphics[width=0.8\textwidth,trim={0 2em 0 3em}]{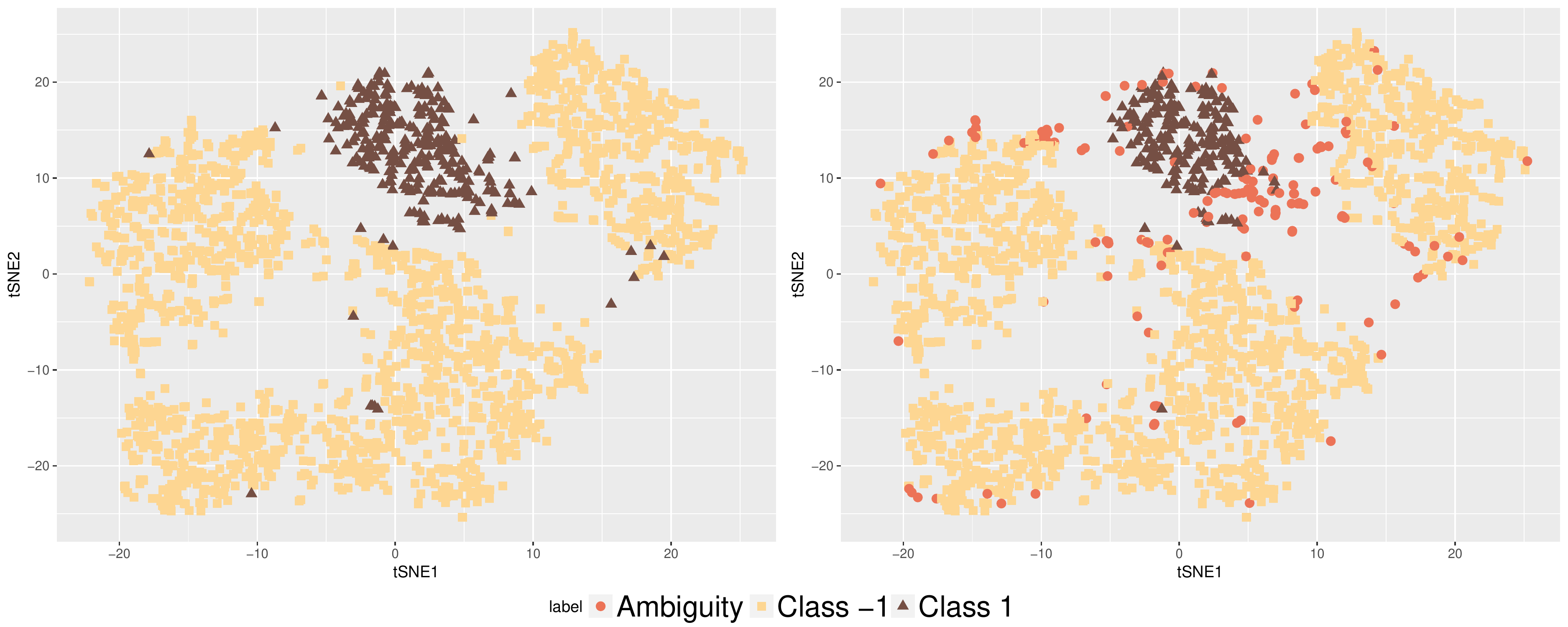}
	\end{center}
	\caption{An illustration of CSVM method using t-SNE. The left penal shows the true labels, and the right panel the predicted label for weighted CSVM.}\label{fig:realdata_illustration}
\end{figure}
\vspace{-0.5cm}
It can be seen that the ambiguity region mainly lies on the boundary between the two classes. In particular, they cover those points which appear to be closer to the class other than the one they really belong to. Moreover, it can be seen that the union of the ambiguity region and the predicted region for either class, covers almost all the ground of that class (defined by the true labels). This is not surprising since the non-coverage rate of CSVM is set to be a small number of 1\% in this case.

\section{Conclusion and future works}\label{conclusion}
In this work, we propose to learn confidence sets using support vector machine. Instead of a plug-in approach, we use empirical risk minimization to train the classifier. Theoretical studies have shown the effectiveness of our approach in controlling the non-coverage rate and minimizing the ambiguity. 

We make use of many well understood advantages of SVM to solve the problem. For instance the `kernel trick' allows more flexibility and empowers us to conduct classification in nonlinear cases.

Hinge loss function is not the only surrogate loss that can be used. There are many other useful loss functions with good properties in different scenarios \cite{liu2011hard}.

Confidence set learning for multi-class case is also an interesting future work. This  has a natural connection to the literature of multi-class classification with confidence \citep{sadinle2017least}, classification with reject and refine options \citep{zhang2017reject} and conformal learning \citep{shafer2008tutorial}. 

\bibliographystyle{plainnat}
\bibliography{LCSuSVM}

\end{document}


\def\spacingset#1{\renewcommand{\baselinestretch}%
{#1}\small\normalsize} \spacingset{1}

\if0\blind
{
  \title{\bf Supplementary Material of "Learning Confidence Sets using Support Vector Machines"}
  \author{Wenbo Wang and Xingye Qiao\thanks{Correspondence to: Xingye Qiao (e-mail: qiao@math.binghamton.edu). Wenbo Wang is a PhD student in the Department of Mathematical Sciences at Binghamton University, State University of New York, Binghamton, New York, 13902 (E-mail: wang2@math.binghamton.edu); and Xingye Qiao is Associate Professor in the Department of Mathematical Sciences at Binghamton University, State University of New York (E-mail: qiao@math.binghamton.edu).}
}
  \maketitle
}\fi

\if1\blind
{
  \bigskip
  \bigskip
  \bigskip
  \begin{center}
    {\LARGE\bf Support Vector Machine with Confidence: Supplementary Material}
  \end{center}
  \medskip
} \fi

\setcounter{page}{1}
\abovedisplayskip=8pt
\belowdisplayskip=8pt

\section{Proof}

This part will give proofs to some of the statements and theorems in the main part.

\subsection*{Proof of Dual Representation}

Firstly, the Lagrangian of problem $(6)$ is 

\begin{align}
\nonumber L(\beta,b,\{\xi_i\}_{i=1}^{n},\{\eta_i\}_{i=1}^{n},\varepsilon)=& \frac{1}{2}||\beta||_{2}^{2}+\lambda' \sum_{i=1}^{n}{\xi_i}+\sum_{i=1}^{n}\zeta_{i}(1+\varepsilon-\xi_i-y_i(\bm{x_i}^{T}\beta+b)) \\
\nonumber &  +\sum_{i=1}^{n}\tau_{i}(1-\varepsilon-\eta_i-y_i(\bm{x_i}^{T}\beta+b)) -\sum_{i=1}^{n}\rho_{i}\xi_i -\sum_{i=1}^{n}\gamma_{i}\eta_i\\
&  + \theta_{-1}(\sum_{y_{i}=-1}{w_{i}\eta_{i}}-n_{-1}\alpha_{-1}) + \theta_{1}(\sum_{y_{i}=1}{w_{i}\eta_{i}}-n_{1}\alpha_{1}) - \nu \varepsilon
\label{lagrangian}
\end{align}

Then we consider the Karush-Kuhn-Tucker conditions. We write $L(\beta,b,\varepsilon,\{\xi_i\}_{i=1}^{n},\{\eta_i\}_{i=1}^{n})$ as $L$ for simplicity.

\begin{align*}
& \frac{\partial L}{\partial \beta}=\beta-\sum_{i}^{n}{\xi_{i}y_{i}\bm{x_i}}-\sum_{i}^{n}{\tau_{i}y_{i}\bm{x_i}}=0 \\
& \frac{\partial L}{\partial b}=-\sum_{i=1}^{n}{\xi_{i}y_{i}}-\sum_{i=1}^{n}{\tau_{i}y_{i}} = 0  \\
& \frac{\partial L}{\partial \xi_i}=\lambda-\xi_{i}-\rho_{i}=0 \quad \textrm{for} \quad \forall i  \\
& \frac{\partial L}{\partial \eta_i}=-\tau_{i}-\gamma_{i}+w_{i}\theta_{-1}=0 \quad \textrm{for} \quad y_{i}=-1 \\
& \frac{\partial L}{\partial \eta_i}=-\tau_{i}-\gamma_{i}+w_{i}\theta_{1}=0 \quad \textrm{for} \quad y_{i}=1 \\
& \frac{\partial L}{\partial \varepsilon}=\sum_{i}^{n}{\zeta_{i}}-\sum_{i}^{n}{\tau_{i}}-\nu=0 \\
& \zeta_{i}(1+\varepsilon-\xi_i-y_i(\bm{x_i}^{T}\beta-b))=0 \quad \textrm{for} \quad i=1,2,...,n \\
& \tau_{i}(1-\varepsilon-\eta_i-y_i(\bm{x_i}^{T}\beta-b))=0 \quad \textrm{for} \quad i=1,2,...,n \\
& \rho_{i}\xi_{i}=0 \quad \textrm{for} \quad i=1,2,...,n \\
& \gamma_{i}\eta_{i}=0 \quad \textrm{for} \quad i=1,2,...,n \\
& \theta_{-1}(\sum_{y_{i}=-1}{\eta_{i}}-n_{-1}\alpha_{-1})=0 \\
& \theta_{1}(\sum_{y_{i}=1}{\eta_{i}}-n_{1}\alpha_{1})=0 \\
& \nu \varepsilon = 0
\end{align*}
After plugging the KKT conditions into expression \ref{lagrangian}, we can get the dual problem. \qed

\subsection*{Proof of Theorem 1}

	In order to prove Theorem 1 and 2, we need to first introduce another risk function, $\bar{R}(f,\varepsilon) =P(Yf(X) < \varepsilon)+\frac{1}{2}P(|f(X)| \leq \varepsilon)$ and the optimization problem associate with it.	
	\begin{align}\label{setcsvm}
	\min ~& \bar{R}(f,\varepsilon) \\
	\nonumber  \textrm{subject to} \quad & P_{j}(Yf(X) < -\varepsilon) \leq \alpha_{j}, \quad j = \pm 1. 
	\end{align}
	And here is a Lemma come with it.
	
	\begin{lemma}\label{equivalentloss1}
		Under Assumption 1, for any fixed $\varepsilon \geq 0$, the discrimination function $f^{*}$ such that 
		$$
		f^{*}(x) =  \left\{\begin{array}{ll}
		1+\varepsilon, & \eta(x) > t_{-1} \\
		\varepsilon*sgn(\eta(x)-\frac{1}{2}), & t_{1} \leq \eta(x) \leq t_{-1} \\
		-(1+\varepsilon), &  f(x) < t_{1}
		\end{array}\right.
		$$
		is a solution to the optimization problem (2) (in the main work) and (\ref{setcsvm}). 
	\end{lemma}

	Denote $C^{*}_{-1}=\{x:f^{*}(x) \leq \varepsilon\}$ and $C^{*}_{1}=\{x:f^{*}(x) \geq -\varepsilon\}$ and the set classifier introduced by $f^{*}$ is $\phi^{*}$. The Assumption 1 ensures that $C^{*}_{-1} \cup C^{*}_{1} = \mathcal{X}$. Let's denote $C^{*}_{-1} \cap C^{*}_{1}$ by $C^{*}_{0}$. 
	
	The optimality of $\phi^{*}$ to Problem (2) (in the main work) is proved in \cite{lei2014classification}. Here we prove the optimality of $f^{*}$ for problem \ref{setcsvm}. The technique used in the following prove is fairly straight-forward in statistical decision and game theory. We start the proof with looking for a so-called complete set of $f$. After that, we only need to focus on this set of discriminant functions. We firstly make two definitions to simplify our proof.
	\begin{definition}
		For any (inequality) constrained optimization problem with m constrains
		\begin{align*}
		& \min{L(f)} \quad \text{such that} \\
	&	C_{i}(f) \leq c_{i} \quad i=1,...,m
		\end{align*} 
		 two function $f_1$ and $f_2$, $f_1$ is said to be as good as than $f_2$ when $L(f_1) \leq L(f_2)$ and $C_{i}(f_{1}) \leq C_{i}(f_{2})$ for $\forall i$, and better than $f_2$ when one of those inequality holds strictly. 
	\end{definition}
	
	\begin{definition}\label{monotonisity}
		Given the distribution of $\bm X$ and $Y$, denoted as $P$. Define a class of function, $\mathcal{F}^{*}(a_{1},a_{2};b_{1},b_{2})$ consists all functions $f$ which take at most two distinct non-negative values $a_1 < a_2$ for $\{x:\eta > \frac{1}{2}\}$ and at most two distinct non-positive values $b_1 > b_2$ for $\{x:\eta(x) < \frac{1}{2}\}$. A constrained optimization problem is said to be simple monotone with respect to $\mathcal{F}^{*}(a_{1},a_{2},b_{1},b_{2})$, if it satisfies:
		
		(i). $\mathcal{F}^{*}(a_{1},a_{2},b_{1},b_{2})$ is a complete class of the problem, which means $\forall f$, $\exists g \in \mathcal{F}^{*}$, and $g$ is as good as $f$. 
		
		(ii). If there exist disjoint $B_{1}, B_{2} \in \mathcal{X}$, such that $P_{-1}(B_{1}) = P_{-1}(B_{2})>0$, for $\forall x_1 \in B_1, x_2 \in B_2$, $\eta(x_1) > \eta(x_2) > \frac{1}{2}$. Moreover, for any pairs of function in $\mathcal{F}^{*}$, $f_{1}(x)$ and $f_{2}(x)$ such that $f_{1}(x) = f_{2}(x)$ for $\forall x \notin B_1 \cup B_2$, and $f_{1}(x) = \left\{\begin{array}{ll}
		a_1, & \forall x \in B_1 \\
		a_2, & \forall x \in B_2 
		\end{array}\right.$ and $f_{2}(x) = \left\{\begin{array}{ll}
		a_2, & \forall x \in B_1 \\
		a_1, & \forall x \in B_2 
		\end{array}\right.$, $f_2$ is better than $f_1$.

		(iii). If there exist disjoint $B_{1}, B_{2} \in \mathcal{X}$, such that $P_{1}(B_{1}) = P_{1}(B_{2})>0$, for $\forall x_1 \in B_1, x_2 \in B_2$, $\eta(x_1) < \eta(x_2) < \frac{1}{2}$. Moreover, for any pairs of function in $\mathcal{F}^{*}$, $f_{1}(x)$ and $f_{2}(x)$ such that $f_{1}(x) = f_{2}(x)$ for $\forall x \notin B_1 \cup B_2$, and $f_{1}(x) = \left\{\begin{array}{ll}
		b_1, & \forall x \in B_1 \\
		b_2, & \forall x \in B_2 
		\end{array}\right.$ and $f_{2}(x) = \left\{\begin{array}{ll}
		b_2, & \forall x \in B_1 \\
		b_1, & \forall x \in B_2 
		\end{array}\right.$, $f_2$ is better than $f_1$.
	\end{definition}
	
	It can be shown that a complete class of a simple monotone optimization problem can be astonishingly simple. We are going to show that we only need to consider the function which depend on $\eta(x)$ rather than $x$. In other words, we can regard $\eta$ as a sufficient statistic of $x$. 
	
	\begin{lemma}\label{sufficientlemma}
		If an optimization $\mathcal{O}$ is simple monotone with respect to $\mathcal{F}^{*}(a_1,a_2;b_1,b_2)$, then a solution of $\mathcal{O}$ in $\mathcal{F}^{*}$ takes the form 
		\begin{align*}
		f(x) =  \left\{\begin{array}{ll}
		a_{2}, & \eta(x) > t \\
		a_{1}, & \frac{1}{2} \leq \eta(x) \leq t \\
		b_{1}, & t' \leq \eta(x) < \frac{1}{2}\\
		b_{2}, &  \eta(x) < t'
		\end{array}\right.
		\end{align*}  for some $t' < \frac{1}{2} \leq t$ almost surely. 
	\end{lemma}
	
	\begin{proof}
		From the simple monotonicity, let's prove there exists a $\frac{1}{2} \leq t \leq 1$ such that $\hat{f}(x)=a_1, \forall x$, such that $\frac{1}{2}<\eta(x)<t, a.s.$ and $\hat{f}(x)=a_2, \forall \eta(x)>t, a.s.$. Define $T_{1} = \{t: \exists C, P(C)>0, \eta(x)>t, \hat{f}(x)=a_1, \forall x \in C \}$ and $T_{2}=\{t: \exists C', P(C')>0, \eta(x)<t, \hat{f}(x) = a_2, \forall x \in C' \}$. Firstly, if $T_{1}=\emptyset$, then $t=\frac{1}{2}$ and similarly, if $T_{2}=\emptyset$, then $t=1$. So now we can assume $T_{1}$ and $T_{2}$ are nonempty. 
		
		If $t_{1} \in T_{1}$, then by definition $t_{2} \in T_{1}, \forall t_{2}<t_{1}$, so that $T_{1}$ is a interval and $\frac{1}{2} \in T_{1}$. Similarly, $T_2$ is also a interval and $1 \in T_{2}$.  Moreover, $T_1$ and $T_2$ are open interval. By definition, $t_1 \in T_1$, then we have $P(C \cap (\cup_n^{\infty} \{x:\eta(x)<t_1+\frac{1}{n}\} )) > 0$, and $\exists m$, such that $P(C \cap  \{x:\eta(x)<t_1+\frac{1}{m}\} ) > 0$. Thus we have $t_1+\frac{1}{m} \in T_{1}$ as well. If $T_{1} \cap T_{2} \neq \emptyset$, then we have a $t' \in T_{1} \cap T_{2}$, which indicates there exists $\forall C_{1}, C_{2} \in \mathcal{X}$ such that $P(C_{1}),P(C_{2})>0$, $1>\eta(x_{1})>\eta(x_{2})>\frac{1}{2}$, and $\hat{f}(x_1) = a_2$, $\hat{f}(x_1) = a_2$, $\forall x_{1} \in C_{1}, x_{2} \in C_{2}$. This will leads to a contradiction with the optimality of $\hat{f}$. If we have $P(C_{1}),P(C_{2})>0$. Then we can choose two subsets of $C_{1}$ and $C_{2}$, $C_{1}'$ and $C_{2}'$ such that $P_{-}(C_{1}')=P_{-}(C_{2}')$, because $P_\eta$ and $P_{-}$ is continuous. If $T_{1} \cap T_{2} = \emptyset$, then we can choose a point t in $[\sup{\{T_1\}},\inf{\{T_2\}}]$ and it will satisfy our purpose.
		
		By similar argument, we can show there exists a $0 \leq t' \leq\frac{1}{2}$ such that $\hat{f}(x)=-(1+\varepsilon), \forall \eta(x)<t', a.s.$ and $\hat{f}(x)=0, \forall \frac{1}{2}>\eta(x)>t', a.s.$.
		
		As a result, the complete set of discriminant functions has the form described in Lemma \ref{sufficientlemma}.
		
	\end{proof}

\textbf{Proof of Lemma \ref{equivalentloss1}}:
	We want to show that \ref{setcsvm} is simple monotone with respect to $\mathcal{F}^{*}(\varepsilon+1,0,0,-(\varepsilon+1))$.
	
	Because optimization problem \ref{setcsvm} can be regarded as an optimization problem for classifiers, it is sufficient to consider functions $f$ with 3 values, $\varepsilon+1$, $-(\varepsilon+1)$, 0, that is $f \in \mathcal{F}(\varepsilon):=\{f:\mathcal{X} \rightarrow \{\varepsilon+1, -(\varepsilon+1), 0\}\}$. 
	
	 Firstly, we need to prove $sign(\eta(X)-\frac{1}{2})\hat{f}(X) \geq 0$ with probability 1 for any $\hat{f}$, a solution of \ref{setcsvm} in $\mathcal{F}(\varepsilon)$. If there is a set $A \subset \mathcal{X}$, $\forall x \in A$, $\eta(x)>\frac{1}{2}$, $\hat{f}=-(\varepsilon+1)$ and $P(A)>0$, then we can consider another function $f_{A}$ such that $f_{A}(x)=\hat{f}(x), \forall x \in A^{c}$ but $f_{A}(x)=0, \forall x \in A$. $f_{A}$ will be better than $\hat{f}$. It is easy to check that two constraints still hold for $f_{A}$. But the objective function will be smaller, because $\frac{1}{2} P(|\hat{f}(X)| \leq \varepsilon)+P(Y\hat{f}(X) < -\varepsilon)-\frac{1}{2} P(|f_{A}(X)| \leq \varepsilon)+P(Yf_{A}(X) < -\varepsilon)=\frac{1}{2} P(|\hat{f}(X)| \leq \varepsilon,X \in A)+P(Y\hat{f}(X) < -\varepsilon,X \in A)-\frac{1}{2} P(|f_{A}(X)| \leq \varepsilon,X \in A)+P(Yf_{A}(X) < -\varepsilon,X \in A)=E(\eta(X)1_{X \in A})>0$. This will lead to a contradiction with the optimality of $\hat{f}$.
		
	We only give the proof for part (ii), and part (iii) can be proved analogously.	We can check the constraints and objective function one by one. 		
	
	Firstly, $P_{+}(Yf_{1}(X)<-\varepsilon) = P_{+}(Yf_{2}(X)<-\varepsilon)$ because the set in which $f_{1}$ and $f_{2}$ take $-(1+\varepsilon)$ are the same. Secondly, $P_{-}(Yf_{1}(X)<-\varepsilon) - P_{-}(Yf_{2}(X)<-\varepsilon)=P_{-}(B_{2})-P_{-}(B_{1})=0$. Lastly, $\frac{1}{2} P(|f_{2}(X)| \leq \varepsilon)+P(Yf_{2}(X) < -\varepsilon)-(\frac{1}{2} P(|f_{1}(X)| \leq \varepsilon)+P(Yf_{1}(X) < -\varepsilon))=E(1_{(X \in B_{2})}(\eta(X)-\frac{1}{2}))-E(1_{(X \in B_{1})}(\eta(X)-\frac{1}{2}))<0$. This comes from the fact that $\eta(x_{1})-\frac{1}{2}>\eta(x_{2})-\frac{1}{2}>0, \forall x_{1} \in B_{1}, x_{2} \in B_{2}$ and $P(B_{2})<P(B_{1})$. The last inequality come from $P(Y=1|X \in B_{1})>P(Y=1|X \in B_{2})$ and $P_{-}(B_{1}) = P_{-}(B_{2})$. 	
	
	Then by Lemma \ref{sufficientlemma}, we can see that the solution of \ref{setcsvm} only depends on $\eta$.
	
	The next part of this proof is to find out the optimal $t$ and $t'$. Let's show that the optimal choice of $t$ is $t_{-1}$. If $t \neq t_{-1}$ for $\hat{f}$, than $t>t_{-1}$ and $P(\eta(x) \leq t) < 1-\alpha_{-1}$, otherwise $\hat{f}$ does not satisfy the constraint that $P_{-}(Yf(X) < -\varepsilon) \leq \alpha_{-1}$. Then If we consider another function $\hat{f}^{*}$ such that $\hat{f}'(x)=0, \forall x, s.t. \quad t_{-1}<\eta(x)<t$ and $\hat{f}^{*}(x)=\hat{f}(x)$ elsewhere. Denote $C'=\{x: t_{-1}<\eta(x)<t\}$ and $P(C')>0$. Then we have that $\frac{1}{2} P(|\hat{f}(X)| \leq \varepsilon)+P(Y\hat{f}(X) < -\varepsilon)-(\frac{1}{2} P(|\hat{f}'(X)| \leq \varepsilon)+P(Y\hat{f}'(X) < -\varepsilon))=E((\frac{1}{2}-(1-\eta(X)))\Ind{(C')})>0$. So $t=t_{-1}$. The optimal choice for $t'$ can be found in a similar way.
	
	The proof is completed by observing $f^{*}$ gives exactly the same $\bar{R}$ loss. 
	 
	 \qed
	
	Now let's start to prove Theorem 1. 
	
	The argument in the proof is similar to Lemma \ref{equivalentloss1}. We are going to show the optimization problem (9) (in the main work) is simply monotone. We consider our proof in two parts. The first is to show the minimizer of optimization problem (9) (in the main work) can only takes four values and is Fisher consistent in a classification sense.
	
	In the first step, let's prove that with probability 1 that $|f^{*}(x)|\leq 1+\varepsilon$. This step is identical to proving the Fisher Consistency of SVM. If a function $f(x)$ has a set $A_1$ with positive probability in $\mathcal{X}$ such that for $\forall x \in \mathcal{X}$, $|f(x)|>1+\varepsilon$, then we can truncate those values to $1+\varepsilon$. In other word, consider $f^{new}(x) = f(x)$ for $x \in {A_1}^{c}$, $f^{new}(x) = (1+\varepsilon)sgn(f(x))$ for $x \in {A_1}$. Then let's prove $f^{new}$ is better than $f$. We can see that the decision implied by $f$ and $f^{new}$ is the same. So the two constrains in \ref{csvm} do not change. However, by looking at the objective function $E[(1+\varepsilon-Yf(X))_{+}] = E[\eta(X)(1+\varepsilon-f(X))_{+}+(1-\eta(X))(1+\varepsilon+f(X))_{+}]$, we can see $f^{new}$ gives smaller loss for all the $X$ such that $\eta_{X} \neq 0, 1$, so that $f^{*}$ will give a smaller expected loss in $A_{1}$.
	
	The next step, we prove $|f^{*}(x)| \geq \varepsilon$ in a similar way. If a function has a set $A_2$ with positive probability in $\mathcal{X}$ such that for $\forall x \in \mathcal{X}$, $|f(x)|<\varepsilon$, then we can enlarge those values of $|f|$ to $\varepsilon sgn(\eta(X)-\frac{1}{2})$. In other words, consider $f^{new}(x) = f(x)$ for $x \in {A_2}^{c}$, $f^{new}(x) = \varepsilon sgn(\eta(X)-\frac{1}{2})$ otherwise. Then let's prove $f^{new}$ is better than $f$. We can see that the decision implied by $f$ and $f^{new}$ is the same. So the two constrains in \ref{csvm} do not change. However, by considering the objective function $E[(1+\varepsilon-Yf(X))_{+}]$ and the result of first step we have $E[(1+\varepsilon-Yf(X))_{+}] = E[(1+\varepsilon-Yf(X))] = E[\eta(X)(1+\varepsilon-f(X))+(1-\eta(X))(1+\varepsilon+f(X))]=E[1+\varepsilon+(1-2\eta(X))f(X)]$. Thus we have $E[H_{\varepsilon}(f)] - E[H_{\varepsilon}(f^{new})] = E[(1-2\eta(X))(f(X)-f^{new}(X))]=E[(1-2\eta(X))(f(X)-\varepsilon sgn(\eta(X)-\frac{1}{2}))1_{X \in A_{2}}]>0$. 
	
	In the third step, we are going to show that $f^{*}$ is Fisher Consistent in the classic classification sense. In other words, $sgn(f^{*}(x)) = sgn(\eta(x)-\frac{1}{2})$ with probability 1. Because of symmetry, let's just prove the case that $\eta(X)>\frac{1}{2}$. If a function has a set $A_3$ with positive probability in $\mathcal{X}$ such that for $\forall x \in \mathcal{X}$, $|f(x)|<0, \eta(x)>\frac{1}{2}$, then we can make them to $\varepsilon$. In other words, consider $f^{new}(x) = f(x)$ for $x \in {A_3}^{c}$, $f^{new}(x) = \varepsilon$ otherwise. Then let's prove $f^{new}$ is more efficient than $f$. The second constraint will not change since $\{x:f(x)>\varepsilon\}=\{x:f^{new}(x)>\varepsilon\}$. The second constraint is also satisfied by $f^{new}$ because we actually have $\{x:f^{new}(x)<-\varepsilon\} \subseteq \{x:f(x)<-\varepsilon\}$. However, $E[H_{\varepsilon}(f)] - E[H_{\varepsilon}(f^{new})]=E[(1-2\eta(X))(f(X)-\varepsilon)1_{X \in A_{3}}]>0$. 
	
	In last step of part one, we want to prove that $f^{*}(x)$ take values between $\varepsilon$ and $\varepsilon+1$ with probability 1. If a function has a set $A_4$ with positive probability in $\mathcal{X}$ such that for $\forall x \in \mathcal{X}$, $\varepsilon<|f(x)|<1+\varepsilon$, then we can enlarge those values of $f$ to $(1+\varepsilon) sgn(\eta(X)-\frac{1}{2})$. In other words, consider $f^{new}(x) = f(x)$ for $x \in {A_2}^{c}$, $f^{new}(x) = (1+\varepsilon) sgn(\eta(X)-\frac{1}{2})$ otherwise. Then let's prove $f^{new}$ is more efficient than $f$. By considering the result of step three, we have the two constraints of $f$ is the same as $f^{new}$, because here we only need to consider the function $f$ such that $sgn(f(x))=sgn(\eta(x)-\frac{1}{2})$. However, $E[H_{\varepsilon}(f)] - E[H_{\varepsilon}(f^{new})]=E[(1-2\eta(X))(f(X)-(1+\varepsilon)sgn(\eta(x)-\frac{1}{2}))1_{X \in A_{4}}]>0$.  
	
	Now we have proved that $f^{*}$ only takes value of $1+\varepsilon$, $\varepsilon$, $-\varepsilon$, $-(1+\varepsilon)$, with probability 1. That is to say $\mathcal{F}^{*}(1+\varepsilon, \varepsilon, -\varepsilon, -(1+\varepsilon))$ is a complete class of the problem
	
	The second part of the proof is to show part (ii) of simple monotonicity. This can be verified by direct calculation which is similar to the proof of Lemma \ref{equivalentloss1}. The last part of this proof is to find out the optimal $t$ and $t'$. The procedure is also analogous to proof of Lemma \ref{equivalentloss1}, thus is omitted here. 
	\qed

\subsection*{Proof of Theorem 2}

We prove this Theorem in two steps. Firstly, we want to use excess risk of $\bar{R} = P(Yf(X) < -\varepsilon)+\frac{1}{2}P(|f(X)|\leq \varepsilon)$ to bound the excess ambiguity $R$. This can be formalized to a Lemma below. 

\begin{lemma}\label{excessriskbound}
	Let $\hat{f}$ be another function that suffices the constraints in (3), then under Assumption 2, for any $\varepsilon \geq 0$, we have $\frac{1}{c}(\bar{R}(\hat{f},\varepsilon)-\bar{R}(f^{*},\varepsilon)) \geq R(\hat{f},\varepsilon)-R(f^{*},\varepsilon)$. 
\end{lemma}

To prove this, we need to further use another lemma which can be regarded as an extension of the theorem before.

\begin{lemma}\label{strongerequivalenttrueloss}
	There $\exists c>0$ satisfies Assumption 2, then for any fixed $\varepsilon \geq 0$, $f^{*}$ is also a solution of the following optimization problem
	\begin{align}
	\textrm{minimize} \quad & (\frac{1}{2}-c) P(|f(X)| \leq \varepsilon)+P(Yf(X) < -\varepsilon) \\
	\nonumber  \textrm{subject to} \quad & P_{j}(Yf(X) < -\varepsilon) \leq \alpha_{j}, \quad j = \pm 1. 
	\end{align}	
\end{lemma}

The proof of this Lemma \ref{strongerequivalenttrueloss} is analogous to the proof of Lemma \ref{equivalentloss1}, thus is omitted here.

By Lemma \ref{strongerequivalenttrueloss}, we have \\
\begin{align}
\nonumber & \frac{1}{c}(\bar{R}(\hat{f},\varepsilon)-\bar{R}(f^{*},\varepsilon)) - (R(\hat{f},\varepsilon)-R(f^{*},\varepsilon)) \\
\nonumber =& \frac{1}{c}(P(Y\hat{f}(X) \leq \varepsilon)-P(Yf^{*}(X) \leq \varepsilon))-(P(|Y\hat{f}(X)| \leq \varepsilon)-P(|Yf^{*}(X)| \leq \varepsilon)) \\
\nonumber =& \frac{1}{c}(((\frac{1}{2}-c) P(|\hat{f}(X)| \leq \varepsilon)+P(Y\hat{f}(X) < -\varepsilon))-(\frac{1}{2}-c) P(|f^{*}(X)| \leq \varepsilon)+P(Yf^{*}(X) < -\varepsilon)) \\
\nonumber \geq &  0 
\end{align} 
\qed

The next step is the prove we can use the excess $R_{H}$ risk to bound the excess risk of $\bar{R}$, which gives the Lemma below.

\begin{lemma}\label{surexcessriskbound2}
	Under Assumption 2, for any $f$ satisfies the constraints in (3), we have 
	\begin{align}
	C \left( R_{H}(f,\varepsilon)-R_{H}(f^{*},\varepsilon) \right) \geq \left(\bar{R}\left(f,\varepsilon \right) - \bar{R}\left(f^{*},\varepsilon \right) \right)
	\end{align}
	where $C = \frac{1}{4c}+\frac{1}{2}$.
\end{lemma}
	
	The proof consists of two steps. First, we will show that we only need to consider the $f$ which takes those values: $1+\varepsilon$, $\varepsilon+$, $\varepsilon$, $-\varepsilon$, $-\varepsilon-$, $-(1+\varepsilon)$. Here $\varepsilon+$ can be regarded as $\varepsilon$ plus a arbitrarily small number and it is similar for $-\varepsilon-$. This can be shown by direct calculation.
	
	Assume $f: \mX \rightarrow \mR$ is an arbitrary discriminate function. Then we consider another function $\bar{f}(\bx)=(1+\varepsilon)\ind{f(\bx)>\varepsilon,\eta(\bx) \geq \frac{1}{2}}+(\varepsilon+)\ind{f(\bx)>\varepsilon,\eta < \frac{1}{2}}+\varepsilon\ind{ |f(\bx)| \leq \varepsilon,\eta \geq \frac{1}{2}}+(-\varepsilon)\ind{|f(\bx)| \leq \varepsilon,\eta < \frac{1}{2}}+(-\varepsilon-)\ind{f(\bx)<-\varepsilon,\eta \geq \frac{1}{2}}+(-(1+\varepsilon))\ind{f(\bx)<-\varepsilon,\eta(\bx) < \frac{1}{2}}$. It is easy to see $\phi_{(f,\varepsilon)}=\phi_{(\bar{f},\varepsilon)}$ so that $\bar{R}(f,\varepsilon)-\bar{R}(f^{*},\varepsilon)=\bar{R}(\bar{f},\varepsilon)-\bar{R}(f^{*},\varepsilon)$. Moreover, by direct calculation, one can show that $R_{H}(f(\bx))\geq R_{H}(\bar{f}(\bx))$ for all $\bx$. So we can see change $f$ to $\bar{f}$ will always leads to a smaller excess surrogate risk while keep the excess risk the same.  
	
	The second part is to explicitly calculate the left hand side and the right hand side and show that the $C$ in the theorem really works. To simplify the notation, we give divide $\mathcal{X}$ by value of $f$(now $f$ take 6 values). For instance, we define $S_{\varepsilon+1}=\{\bx:f(\bx)=\varepsilon+1 \}$ and by the first part, we can assume $\eta(S_{\varepsilon+})<\frac{1}{2}$ and $\eta(S_{-\varepsilon-})>\frac{1}{2}$. To ease the notation, we omit the independent variable $X$ in following expressions although the expectation is really taken with respect to it. Then we have 

\begin{align*}
R_{H}(f)-R_{H}(f^*) &= E(\ind{S_{\varepsilon+1}}(2(1+\varepsilon)(1-\eta)))+E(\ind{S_{\varepsilon}}(1+2\varepsilon(1-\eta))) \\
&+ E(\ind{S_{-\varepsilon-}}(1+2\varepsilon \eta))+E(\ind{S_{\varepsilon+}}(1+2\varepsilon(1-\eta))) \\
&+ E(\ind{S_{-\varepsilon}}(1+2\varepsilon \eta))+E(\ind{S_{-(\varepsilon+1)}}(2(1+\varepsilon)\eta)) \\
&- E(\ind{\frac{1}{2} \leq \eta \leq t_{-1}}(1+2\varepsilon(1-\eta)))-E(\ind{t_{1} \leq \eta \leq \frac{1}{2}}(1+2\varepsilon \eta)) \\
&- E(\ind{ \eta > t_{-1}}(2(1+\varepsilon)(1-\eta)))-E(\ind{\eta < t_{1}}(2(1+\varepsilon)\eta))
\end{align*}
and
\begin{align*}
R(f)-R(f^*) &= E(\ind{S_{\varepsilon+1}}(1-\eta))+E(\ind{S_{\varepsilon}}(\frac{1}{2})) 
+ E(\ind{S_{-\varepsilon-}}(\eta) \\
&+E(\ind{S_{\varepsilon+}}(1-\eta)) + E(\ind{S_{-\varepsilon}}(\frac{1}{2}))+E(\ind{S_{-(\varepsilon+1)}}(\eta)) \\
&- E(\ind{\frac{1}{2} \leq \eta \leq t_{-1}}(\frac{1}{2}))-E(\ind{t_{1} \leq \eta \leq \frac{1}{2}}(\frac{1}{2})) \\
&- E(\ind{ \eta > t_{-1}}(1-\eta))-E(\ind{\eta < t_{1}}(\eta))
\end{align*}
Then by some algebra, we have $C(R_{H}(f)-R_{H}(f^{*})) \geq R(f)-R(f^*)$ is equivalent to 
$A+2\varepsilon C B \geq 0$ where 
\begin{align*}
A &= E(\ind{S_{\varepsilon+1}}((2C-1)(1-\eta)))+E(\ind{S_{\varepsilon}}(C-\frac{1}{2})) \\
&+ E(\ind{S_{-\varepsilon-}}(C-\eta))+E(\ind{S_{\varepsilon+}}(C-(1-\eta))) \\
&+ E(\ind{S_{-\varepsilon}}(C-\frac{1}{2}))+E(\ind{S_{-(\varepsilon+1)}}((2C-1)\eta)) \\
&- E(\ind{\frac{1}{2} \leq \eta \leq t_{-1}}(C-\frac{1}{2}))-E(\ind{t_{1} \leq \eta \leq \frac{1}{2}}(C-\frac{1}{2})) \\
&- E(\ind{ \eta > t_{-1}}((2C-1)(1-\eta)))-E(\ind{\eta < t_{1}}((2C-1)\eta))
\end{align*}
and $B = P(f(\bX)Y<0)-P(f^{*}(\bX)Y<0)$. By the definition of $f^{*}$ we can easily see that $B \geq 0$. So the rest is to show $A \geq 0$. We can only focus on $C>\frac{1}{2}$. Divide A by $2C-1$ and do some algebra, we have $A \geq 0$ is equivalent to
\begin{align*}
& (E(\ind{S_{\varepsilon+1}}((1-\eta))+E(\ind{S_{\varepsilon}}(\frac{1}{2})+E(\ind{S_{-\varepsilon-}}(\frac{1}{2})\\
&+ E(\ind{S_{\varepsilon+}}(\frac{1}{2}))+E(\ind{S_{-\varepsilon}}(\frac{1}{2})+E(\ind{S_{-(\varepsilon+1)}}(\eta))) \\
&- (E(\ind{\frac{1}{2} \leq \eta \leq t_{-1}}(\frac{1}{2}))+E(\ind{t_{1} \leq \eta \leq \frac{1}{2}}(\frac{1}{2}) \\
&+ E(\ind{ \eta > t_{-1}}(1-\eta)-E(\ind{\eta < t_{1}}(\eta)))  \\
&\geq \frac{1}{2C-1}(E(\ind{S_{-\varepsilon-}}(\eta-\frac{1}{2})
+ E(\ind{S_{\varepsilon+}}(\frac{1}{2}-\eta)))
\end{align*}
It is not hard to see the first part of the left hand side is a $\bar{R}$ risk of a classifier with $+1$ prediction at $S_{\varepsilon+1}$, negative prediction at $S_{-(\varepsilon+1)}$ and ambiguity else where. The second part is the risk of $f^{*}$. By definition of $f^{*}$, we have $P_{-1}(\eta \leq t_{-1})=1-\alpha_{-1}$, $P_{1}(\eta \geq t_{1})=1-\alpha_{1}$. Let $\alpha'_{-1}=\alpha_{-1}-P_{-1}(S_{\varepsilon+})$ and $\alpha'_{1}=\alpha_{1}-P_{1}(S_{-\varepsilon-})$ and let $t'_{-1}$ and $t'_{1}$ satisfy  $P_{-1}(\eta \leq t'_{-1})=1-\alpha'_{-1}$, $P_{1}(\eta \geq t'_{1})=1-\alpha'_{1}$. Because $\eta(S_{\varepsilon+})<\frac{1}{2}$, we have $P(t_{-1}<\eta \leq t'_{-1})>P(S_{\varepsilon+})$ by Bayes Formula. Similarly $P(t'_{1}\leq \eta < t_{1})>P(S_{-\varepsilon-})$. So at last, we have
\begin{align*}
 \text{LHS of above} &\geq E(\ind{t_{-1}<\eta  \leq t'_{-}}(\eta-\frac{1}{2}))+E(\ind{t'_{+}\leq \eta < t_{1}}(\frac{1}{2}-\eta))  \\
&\geq c(P(t_{-1}<\eta \leq t'_{-})+P(t'_{+}\leq \eta < t_{1}))\geq c(P(S_{\varepsilon+})+P(S_{-\varepsilon-})) \\
&= \frac{1}{2C-1}\frac{1}{2}(P(S_{\varepsilon+})+P(S_{-\varepsilon-})) \\
&\geq \frac{1}{2C-1}(E(\ind{S_{-\varepsilon-}}(\eta-\frac{1}{2})
+ E(\ind{S_{\varepsilon+}}(\frac{1}{2}-\eta)))
\end{align*}
So we have $A \geq 0$ thus the statement of our theorem holds.

Note that one can induce a small $\delta$, i.e., using $\varepsilon+\delta$ instead of using the notation $\varepsilon+$ and let $\delta$ goes to 0 at the end of the proof to make it more rigorous. However, because there is no limit involved in other parts of this proof, we can live with this notation to keep us from those trouble. 
\qed

Lastly, Theorem 2 is a direct corollary of Lemma \ref{excessriskbound} and Lemma \ref{surexcessriskbound2}.
\qed

\subsection*{Proof of Theorem 3}

	To prove this theorem, we need to introduce Rademacher complexity which has been widely used in statistical machine learning theory.  

Here we only prove inequality for $Y=-1$, the proof for $Y=1$ case can be down analogously. Without loss of generality, we assume the first $n_{-1}$ observations are from $-1$.	
	
	Let $\sigma = \{\sigma_{i};i = 1,...,n_{-1}\}$ be independent and identically distributed random variables from discrete uniform distribution U(\{-1,1\}). Also denote by $S$ a sample of observations $(\bm{x_i},y_i)$; i=1,...,$n_{-1}$, independent and identically distributed from the underlying distribution $P(\bm{X},Y|Y=-1)$ (Y will always be $-1$ in this case). we define the empirical Rademacher complexity of the function class with fixed b, $\mathcal{H}^{b}_{K}(s) = \{h(x)+b|h \in \mathcal{H}_{K}, ||h||_{\mathcal{H}_{K}} \leq s\}$ as follows,
	\begin{align}
	\hat{R}_{n_{-1}}\{\mathcal{H}^{b}_{K}(s)\} = E_{\sigma}[\sup_{f \in \mathcal{H}^{b}_{K}(s)}{\frac{1}{n_{-1}}\sum_{i=1}^{n_{-1}}\sigma_{i}H_{-\varepsilon}(y_{i}f(\bm{x_i}))}]
	\end{align}
	Here $E_{\sigma}$ means taking expectation with respect to the joint distribution of $\sigma$. Moreover, we can define the Rademacher complexity of $\mathcal{H}^{b}_{K}(s)$ to be 
	\begin{align}
	R_{n_{-1}}\{\mathcal{H}^{b}_{K}(s)\} = E_{\sigma,S}[\sup_{f \in \mathcal{H}^{b}_{K}(s)}{\frac{1}{n_{-1}}\sum_{i=1}^{n_{-1}}\sigma_{i}H_{-\varepsilon}(y_{i}f(\bm{X_i}))}]
	\end{align}
	where S is the sample space given $Y=-1$. 
	
	The next step is to construct the standard inequality of Rademacher complexity. It controls the expected hinge loss for negative group by the summation of empirical hinge loss, empirical Rademacher complexity and a small penalty term, which can be summarized in the following lemma. This lemma is important and will be used in the proves follows.
		
	\begin{lemma}\label{rademacherlemma}
		Let $\hat{R_{n}}\{\mathcal{H}^{b}_{K}(s)\}$ and $R_{n}\{\mathcal{H}^{b}_{K}(s)\}$ be defined as above. Then with probability at least $1-\zeta$, 
		\begin{align}\label{firstrade}
		E(H_{-\varepsilon}(Yf(\bm{X}))) \leq \frac{1}{n_{-1}}\sum_{i=1}^{n_{-1}}{H_{-\varepsilon}(y_{i}f(\bm{x_i}))} + 2R_{n_{-1}}\{\mathcal{H}^{b}_{K}(s)\} + T_{n_{-1}}(\zeta),
		\end{align}		
		Moreover, with probability at least $1-\zeta$, 
		\begin{align}\label{secondrade}
		E(H_{-\varepsilon}(Yf(\bm{X}))) \leq \frac{1}{n_{-1}}\sum_{i=1}^{n_{-1}}{H_{-\varepsilon}(y_{i}f(\bm{x_i}))} + 2\hat{R}_{n_{-1}}\{\mathcal{H}^{b}_{K}(s)\} + 3T_{n_{-1}}(\zeta/2).
		\end{align}
	\end{lemma}

	\begin{proof}

	 The proof consist of three parts. In the first part, we use the McDiarmid inequality to bound the left hand side of inequality \ref{firstrade} by its empirical counterpart and $\phi(S)$ which is define below:
	\begin{align*}
	\phi(S) = \sup_{f \in \mathcal{H}^{b}_{K}(s)}{\{E(H_{-\varepsilon}(Yf(X))) - \frac{1}{n_{-1}}\sum_{i=1}^{n_{-1}}H_{-\varepsilon}(y_{i}f(\bm{x_{i}}))\}}
	\end{align*}
	
	Let $S^{(i)} = \{(\bm{x_1},y_1),...(\bm{x_i}{'},y_i),...(\bm{x_n},y_n)\}$ be another sample from $P(\bm{X},Y|Y = -1)$,  where the difference between $S$ and $S^{(i,x)}$ is just the $i$th observation. Then by definition, we have
	\begin{align*}
	\nonumber \phi(S)-\phi(S^{(i)}) &= \sup_{f \in \mathcal{H}^{b}_{K}(s)}{\{E(H_{-\varepsilon}(Yf(X))) - \frac{1}{n_{-1}}\sum_{S}H_{-\varepsilon}(y_{i}f(\bm{x_{i}}))\}} \\
	\nonumber &- \sup_{f \in \mathcal{H}^{b}_{K}(s)}{\{E(H_{-\varepsilon}(Yf(X))) - \frac{1}{n_{-1}}\sum_{S^{i,\bm{x}}}H_{-\varepsilon}(y_{i}f(\bm{x_{i}}))\}}.
	\end{align*}
	
	Note that it is easy to show the difference of supremum of two functions is smaller than the supremum of the difference of two functions. 
	
	Then we have 
	\begin{align}
	\nonumber \phi(S)-\phi(S^{(i)}) &\leq  \sup_{f \in \mathcal{H}^{b}_{K}(s)}{\{E(H_{-\varepsilon}(Yf(X))) - \frac{1}{n_{-1}}\sum_{S}H_{-\varepsilon}(y_{i}f(\bm{x_{i}}))\}} \\
	\nonumber &- \{E(H_{-\varepsilon}(Yf(X))) - \frac{1}{n_{-1}}\sum_{S^{i,\bm{x}}}H_{-\varepsilon}(y_{i}f(\bm{x_{i}}))\} \\
	\nonumber &= \sup_{f \in \mathcal{H}^{b}_{K}(s)} {\{\frac{1}{n_{-1}} H_{-\varepsilon}(y_{i}f(\bm{x_i}))-H_{-\varepsilon}(y_{i}f(\bm{x_i}'))\}} \\
	\nonumber &\leq  \sup_{f \in \mathcal{H}^{b}_{K}(s)} {\{\frac{1}{n_{-1}} |\{f(\bm{x_i})-f(\bm{x_i}')|\}}  \\
	\nonumber &\leq  \sup_{h \in \mathcal{H}_{K}, ||h||_{\mathcal{H}_{K}} \leq s} {\{\frac{1}{n_{-1}} |\{h(\bm{x_i})-h(\bm{x_i}')|\}} \\
	\nonumber &\leq  \sup_{h \in \mathcal{H}_{K}, ||h||_{\mathcal{H}_{K}} \leq s} {\{\frac{1}{n_{-1}} |\langle h,K(\bx_{i},\cdot) \rangle - \langle h,K(\bx_{i}',\cdot ) \rangle |\}} \\
	\nonumber &\leq  \frac{1}{n_{-1}} \sup_{h \in \mathcal{H}_{K}, ||h||_{\mathcal{H}_{K}} \leq s} {\{ |\langle h,K(\bx_{i},\cdot) \rangle | + |\langle h,K(\bx_{i}',\cdot ) \rangle |\}} \\
	\nonumber &\leq  \frac{2}{n_{-1}} \sup_{h \in \mathcal{H}_{K}, ||h||_{\mathcal{H}_{K}} \leq s, \bx \in \mathcal{X}} {\{ \sqrt{||h||_{\mathcal{H}_{K}}||K(\bx,\bx)||\}}} \\
	\nonumber &\leq  \frac{2\sqrt{sr}}{n_{-1}}
	\end{align}
	
	Because $S$ and $S^{i}$ are symmetric, as a result, we have $|\phi(S)-\phi(S^{(i)})| \leq \frac{2\sqrt{sr}}{n_{-1}}$.
	
	Next, by the McDiarmid inequality, we have that for any $t>0$, $P(\phi - E(\phi(S))\geq t)\leq exp(-\frac{t^2n_{-1}}{2sr} )$, or equivalently, with probability $1-\zeta$, $\phi(S) - E(\phi(S))\leq T_{n}(\zeta)$. Consequently, we have that with probability at least $1-\zeta$, $E(H_{-\varepsilon}(Yf(\bm{X}))) \leq \frac{1}{n}\sum_{i=1}^{n}{H_{-\varepsilon}(y_{i}f(\bm{x_i}))}+E\{\phi(S)\}+T_{n_{-1}}(\zeta)$. This gives the first part of the proof. 
	
	In the second part, we need to bound $E\{\phi(S)\}$ by the Rademacher complexity. Define $S' = \{(x'_{i},y'_{i});i=1,...,n_{-1}\}$ as an independent identical duplicate of $S$. Then we have that
	\begin{align}
	\nonumber E\{\phi(S)\} &= E_{S}(\sup_{f \in \mathcal{H}^{b}_{K}(s)}{E_{S'}[\frac{1}{n}\sum_{S'}{H_{-\varepsilon}(y'_{i}f(\bm{x'_{i}}))}-\frac{1}{n}\sum_{S}{H_{-\varepsilon}(y_{i}f(\bm{x_{i}}))}]}|S) \\
	\nonumber &\leq  E_{S,S'}[\frac{1}{n}\sum_{S'}{H_{-\varepsilon}(y'_{i}f(\bm{x'_{i}}))}-\frac{1}{n}\sum_{S}{H_{\varepsilon}(y_{i}f(\bm{x_{i}}))}] \\
	\nonumber &= E_{S,S',\sigma}[\frac{1}{n}\sum_{S'}{\sigma_{i}H_{-\varepsilon}(y'_{i}f(\bm{x'_{i}}))}-\frac{1}{n}\sum_{S}{\sigma_{i}H_{\varepsilon}(y_{i}f(\bm{x_{i}}))}] \\
	\nonumber &\leq  2R_{n}\{\mathcal{H}^{b}_{K}(s)\}
	\end{align}
	
	Combining the first and second step, we have already proved first inequality in Lemma \ref{rademacherlemma}. 
	
	The third step is analogous to the first step. We will use the empirical Rademacher complexity to bound the population Rademacher complexity. 
	
	This can be shown by defining $\psi(S) = \hat{R_{n}}\{\mathcal{H}^{b}_{K}(s)\}$ and it is easy to see $|\psi(S)-\psi(S')| \leq \frac{2\sqrt{sr}}{n_{-1}}$ by the definition of empirical Rademacher complexity. Then we can use McDiarmid inequality again and get with probability at least $1-\zeta$, $\psi(S)-E(\psi(S)) \leq T_{n}(\zeta)$. At last, we can combine this outcome and \ref{firstrade} by choose the confidence level to be $1-\zeta/2$ to get \ref{secondrade}.
	\end{proof}
	
	Then last step will be controlling the empirical Rademacher complexity for kernel learning. In particular, by Lemma 4.2 and Theorem 5.5 in \cite{mohri2012foundations}, we can have that $\hat{R_{n}}\{\mathcal{H}^{b}_{K}(s)\}$ can be upper bounded by the following inequality
	
	\begin{align*}
	\hat{R_{n}}\{\mathcal{H}^{b}_{K}(s)\} &\leq  E_{\sigma}[\sup_{h \in \mathcal{H}_{K}, ||h||_{\mathcal{H}_{K}} \leq s}{\frac{1}{n}\sum_{i=1}^{n}{\sigma_i h(\bm{x_i})}}] \\
	& \leq \frac{rs}{\sqrt{n}}
	\end{align*}
\qed

\subsection*{Proof of Theorem 4}
This proof is similar to proof of Theorem 5 in \cite{rigollet2011neyman}.

Statement (a) of this theorem is the direct Corollary of Theorem 3. One can see the proof for Lemma \ref{rademacherlemma} does not only work for $H_{-\varepsilon}$, but also work for $H_{c}$ with any $c$. In particular, it works for $\varepsilon$. So define the events $E_{1}$ and $E_{2}$. Let $R^{j}_{H}(f,c) = E(H_{c}(Yf(X))|Y=j)$, and $\hat{R}^{j}_{H}(f,c)$ be its empirical counterpart for $j = \pm 1$. 
\begin{align*}
E^{-}(f,\varepsilon) &= \{|\hat{R}^{-1}_{H}(f,-{\varepsilon}) - R^{-1}_{H}(f,-{\varepsilon})| \leq \frac{\kappa}{\sqrt{n_{-1}}},|\hat{R}^{1}_{H}(f,-{\varepsilon}) - R^{1}_{H}(f,-{\varepsilon})| \leq \frac{\kappa}{\sqrt{n_{1}}}\} \\    
E^{+}(f,\varepsilon) &= \{|\hat{R}^{-1}_{H}(f,{\varepsilon}) - R^{-1}_{H}(f,{\varepsilon})| \leq \frac{\kappa}{\sqrt{n_{-1}}},|\hat{R}^{1}_{H}(f,{\varepsilon}) - R^{1}_{H}(f,{\varepsilon})| \leq \frac{\kappa}{\sqrt{n_{1}}}\}          
\end{align*}
By Theorem 3, $P(E) \geq 1-2\zeta$ for any fix $f \in \mathcal{H}_{K}(s)$ and $c \in \mathbb{R}$. 

To study statement (b), we can decompose the left hand side of the inequality into three parts and study them one by one. To simplify notations, we are going to omit norm upper bound $s$ in notation of $\mathcal{F}$. For instance, we will write $\mathcal{F}_{\varepsilon}(0)$ instead of $\mathcal{F}_{\varepsilon}(0,s)$.
$$R_{H}(\hat{f},\hat{\varepsilon})-\min_{(f,\varepsilon) \in \mathcal{F}(0)}{R_{H}(f,{\varepsilon})} = A_{1}+A_{2}+A_{3}$$

Where 
\begin{align*}
A_{1} &= (R_{H}(\hat{f},\hat{\varepsilon})-\hat{R}_{H}(\hat{f},\hat{\varepsilon})) + (\hat{R}_{H}(\hat{f},\hat{\varepsilon}) - \min_{(f,\varepsilon) \in \hat{\mathcal{F}}(\kappa)}{R_{H}(f,{\varepsilon})}) \\
A_{2} &= \min_{(f,\varepsilon) \in \hat{\mathcal{F}}(\kappa)}{R_{H}(f,{\varepsilon})} - \min_{(f,\varepsilon) \in \mathcal{F}(2\kappa)}{R_{H}(f,{\varepsilon})}                 \\
A_{3} &= \min_{(f,\varepsilon) \in \mathcal{F}(2\kappa)}{R_{H}(f,{\varepsilon})} - \min_{(f,\varepsilon) \in \mathcal{F}(0)}{R_{H}(f,{\varepsilon})}
\end{align*}

Because we only focus on $E_1 = E^{+}(\hat{f},\hat{\varepsilon}) \bigcap E^{+}(\argmin_{(f,\varepsilon) \in \hat{\mathcal{F}}(\kappa)}{R_{H}(f,\varepsilon)})$, then we have $$
A_{1} \leq \frac{2\kappa}{\sqrt{n}}
$$
It is easy to see that $A_{2} \leq 0$ for large enough n on $E_2 = E^{-}(\argmin_{(f,\varepsilon) \in {\mathcal{F}}(2\kappa)}{R_{H}(f,\varepsilon)})$.

The last part of the proof is to bound $A_3$. To begin with the proof, let's first introduce a lemma. 

\begin{lemma}\label{convexlemma}
Let $\gamma_{s}((\alpha_{-1},\alpha_{1})) = $ be a function from $[0,1]^2$ to $\inf\limits_{f \in \mathcal{F}(0,s)}{R_{H}(f,\varepsilon)}$. Then $\gamma_{\varepsilon}$ is convex in $[0,1]^2$. Moreover, $\gamma_{s}((\alpha_{-1},\alpha_{1})) \leq \gamma_{s}((\alpha_{-1}',\alpha_{1}'))$ for $\alpha_{-1} \geq \alpha_{-1}'$ and $\alpha_{1} \geq \alpha_{1}'$.
\end{lemma}	

\begin{proof}
By the convexity of loss function $H_{c}$, we have $E(H_{(\theta c_1 +(1-\theta)c_2)}(Y(\theta f_{1}(X)+(1-\theta) f_{2}(X)))) \leq \theta E(H_{c_1}(Yf_{1}(X))) + (1-\theta)E(H_{c_2}(Yf_{2}(X)))$ for all $\theta \in [0,1]$. By definition of infimum, for any $\mu > 0$ and $\alpha^{1} = (\alpha_{-1},\alpha_{1})$, there exists a $f_{1} \in \mathcal{F}(0,s)$, such that $\gamma_s(\alpha^1) > E(H_{\varepsilon}(Yf_{1}(X))) - \mu$, and for another $\alpha^2$, there exists a $f_2$ as well.

By the argument above, we have $\gamma_{s}(\theta \alpha^{1}+(1-\theta)\alpha^{2}) \leq E(H_{(\theta \varepsilon_1 +(1-\theta)\varepsilon_2)}(Y(\theta f_{1}(X)+(1-\theta) f_{2}(X)))) \leq \theta E(H_{\varepsilon_1}(Yf_{1}(X))) + (1-\theta)E(H_{\varepsilon_2}(Yf_{2}(X))) \leq \theta \gamma_{s}(\alpha^1) + (1-\theta)\gamma_{s}(\alpha^2) - \mu$ for all positive $\mu$. And it is easy to verify that $\theta f_{1}+(1-\theta) f_{2}$ and $(\theta \varepsilon_1 +(1-\theta)\varepsilon_2)$ give satisfy the constraints. So that $\gamma_{s}$ is convex. 

The second statement of the lemma is easy to see by noticing that $\mathcal{H}_{K}(s) \cap \mathcal{F}_{\varepsilon}(\alpha_{-1}',\alpha_{1}') \subset \mathcal{H}_{K}(s) \cap \mathcal{F}_{\varepsilon}(\alpha_{-1},\alpha_{1})$ for $\alpha_{-1} \geq \alpha_{-1}'$ and $\alpha_{1} \geq \alpha_{1}'$.
\end{proof}

The last part of the proof is from the convexity of $\gamma_s$. For a large enough $n_{-1}$ and $n_{1}$, we will finally have $\frac{\kappa}{\sqrt{n_{-1}}} < \alpha_{-1},\frac{\kappa}{\sqrt{n_{1}}} < \alpha_{1})$. Let $\nu = (\frac{\kappa}{\sqrt{n_{-1}}},\frac{\kappa}{\sqrt{n_{1}}})$.

Now by convexity of $\gamma_s$, we have $\gamma_{s}(\alpha) - \gamma_{s}(\alpha-\nu) \geq \nu \cdot g$ and $\gamma_{s}(\alpha-\nu_{0}) - \gamma_{s}(\alpha-\nu) \geq (\nu-\nu_{0}) \cdot g$, where g in any member of the subgradiant of $\gamma$ at $\alpha-\nu$. After combining these two inequalities, we have
\begin{align*}
\gamma_{s}((0,0)) - \gamma_{s}(\alpha-\nu) &\geq  (\alpha-\nu) \cdot (-g) \\
&\geq  \min{\{\frac{\alpha_{-1}-\frac{\kappa}{\sqrt{n_{-1}}}}{\frac{\kappa}{\sqrt{n_{-1}}}},\frac{\alpha_{1}-\frac{\kappa}{\sqrt{n_{1}}}}{\frac{\kappa}{\sqrt{n_{1}}}}\}} \nu \cdot (-g) \\
&\geq  \min{\{\frac{\alpha_{-1}-\frac{\kappa}{\sqrt{n_{-1}}}}{\frac{\kappa}{\sqrt{n_{-1}}}},\frac{\alpha_{1}-\frac{\kappa}{\sqrt{n_{1}}}}{\frac{\kappa}{\sqrt{n_{1}}}}\}} (\gamma_{s}(\alpha-\nu) - \gamma_{s}(\alpha))
\end{align*} 
This will finally lead us to $\gamma_{s}(\alpha-\nu) - \gamma_{s}(\alpha) \leq (\gamma_{s}(\alpha) - \gamma_{s}(\alpha-\nu))\frac{2\max{\{\frac{\kappa}{\sqrt{n_{-1}}},\frac{\kappa}{\sqrt{n_{1}}}}\}}{\min{\{\alpha_{-1},\alpha_{1}\}}} \leq \frac{2\kappa}{\min{\{\alpha_{-1},\alpha_{1}\}}\min{\{\sqrt{n_{-1}},\sqrt{n_{1}}\}}}$. The last inequality is directly from the fact that $f(x) \equiv 0$ and $\varepsilon = 0$ satisfies the constraints of problem (5) and gives a loss 1 (in the main work). Thus we have $A_{3} \leq \frac{4\kappa}{\min{\{\alpha_{-1},\alpha_{1}\}}\min{\{\sqrt{n_{-1}},\sqrt{n_{1}}\}}}$.

Then the proof is finished by combining $A_1, A_{2}, A_{3}$. \qed

\section{More on numerical study}

For each simulation scenario, we give a plot of non-coverage rates for both $-1$ and $1$ class. We also give plots of the proportion of instances in which both classes have the desired test non-coverage rates, e.g. 0.05 or smaller. 

\subsection*{Linear model with nonlinear Bayes rule}

\begin{figure}[H]
	\includegraphics[height = 6cm,width=1\textwidth]{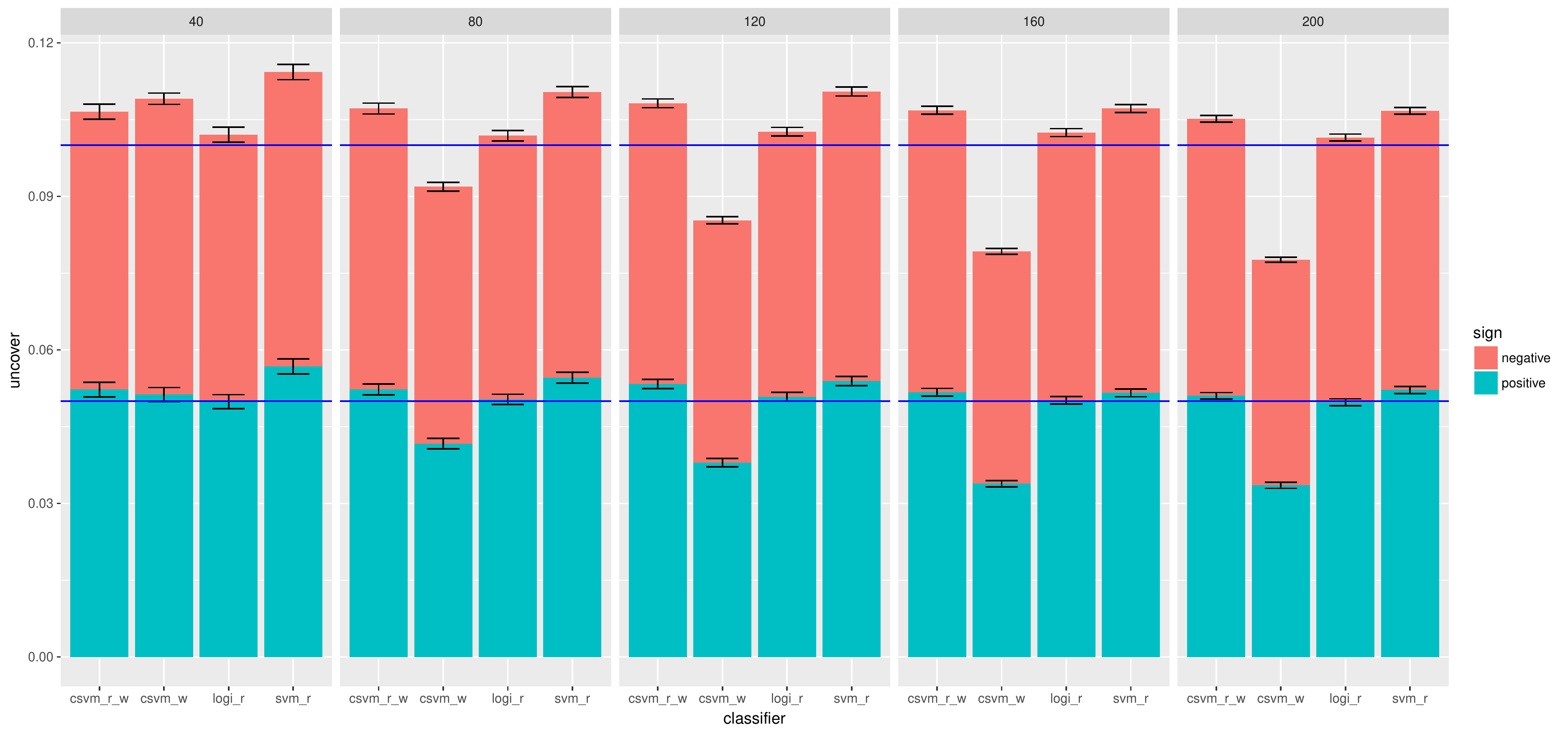}
	\caption{Non-coverage rates for all the models. We can see that weighted CSVM has a smaller non-coverage rates when sample size become larger, which explains why it has a relatively larger ambiguity. It worth to note that when $n=80$, weighted CSVM has a significantly smaller non-coverage rates than plug-in methods and maintain a smaller (or comparable) ambiguity.}
\end{figure}
\begin{figure}[H]
	\includegraphics[height = 6cm,width=1\textwidth]{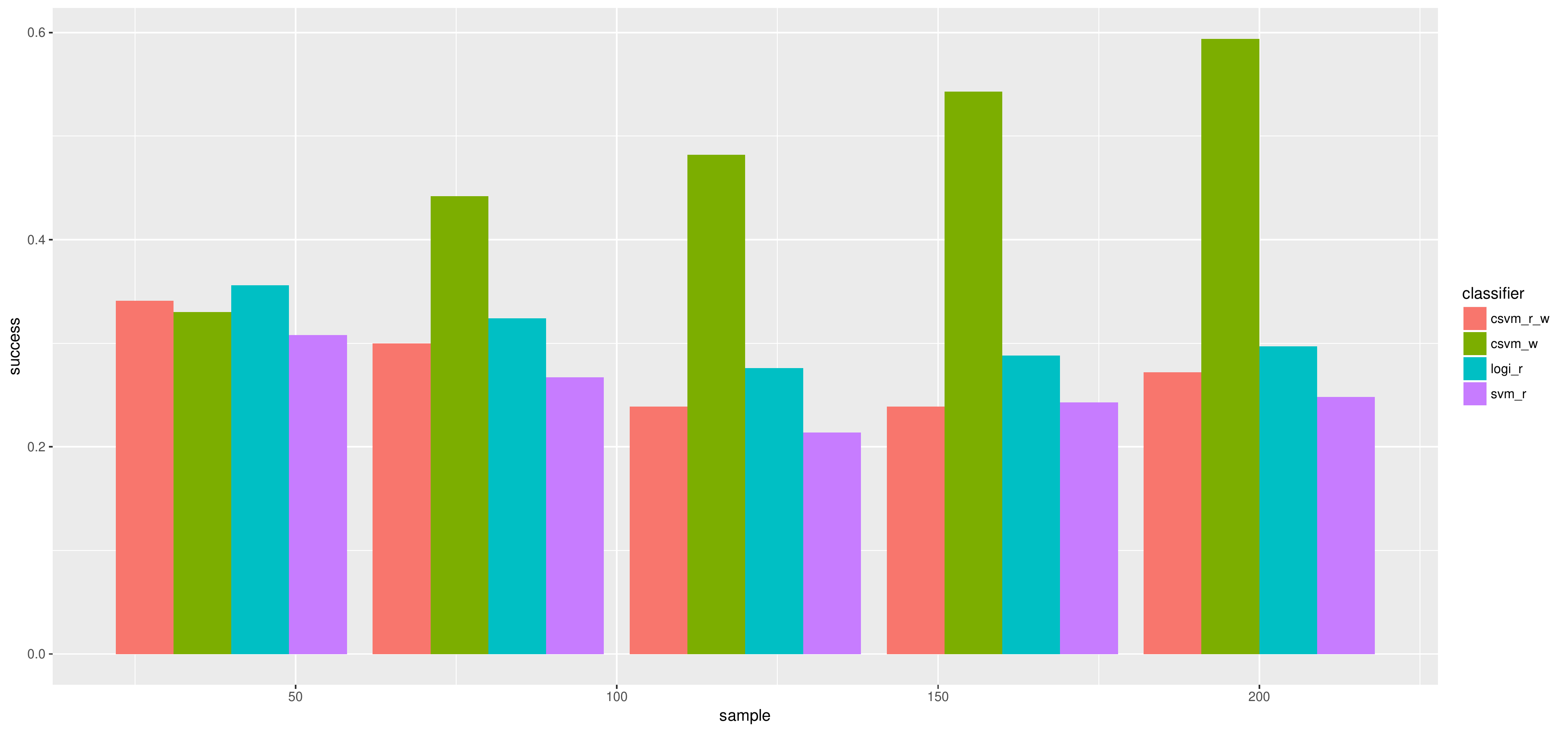}
	\caption{Success (to cover desired observations) rates for all the models. We can see that weighted CSVM has a greater success rates when sample size become larger, which also explains why it has a relatively larger ambiguity. It worth to note that when $n=80$, weighted CSVM has a much larger non-coverage rates than plug-in methods and maintain a smaller (or comparable) ambiguity.}
\end{figure}

\subsection*{Moderate dimensional polynomial boundary}

\begin{figure}[H]
	\includegraphics[height = 6cm,width=1\textwidth]{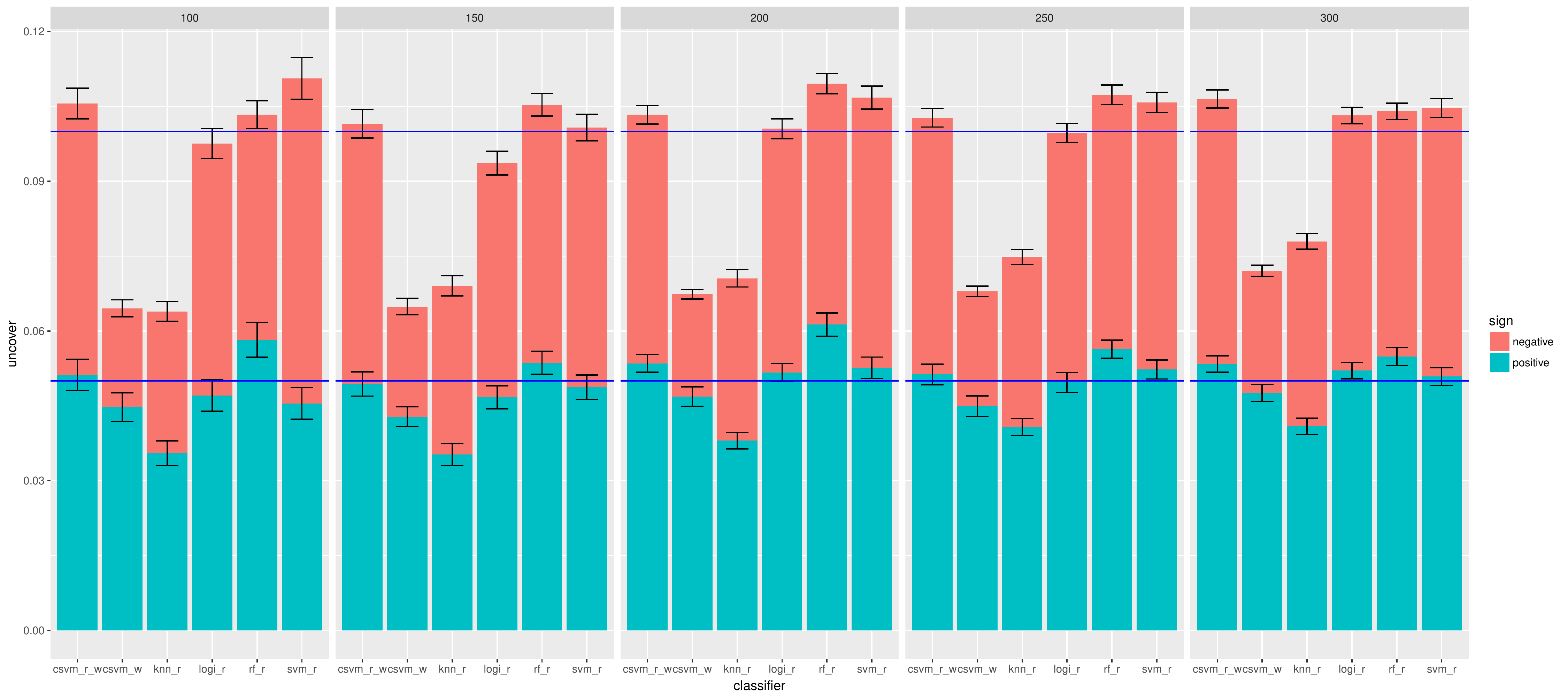}
	\caption{Non-coverage rates for all the models. We can see that weighted CSVM and kNN has a smaller non-coverage rates when the other three have similar non-coverage rates. But within those two groups, the proposed model always has a smaller ambiguity.}
\end{figure}
\begin{figure}[H]
	\includegraphics[height = 6cm,width=1\textwidth]{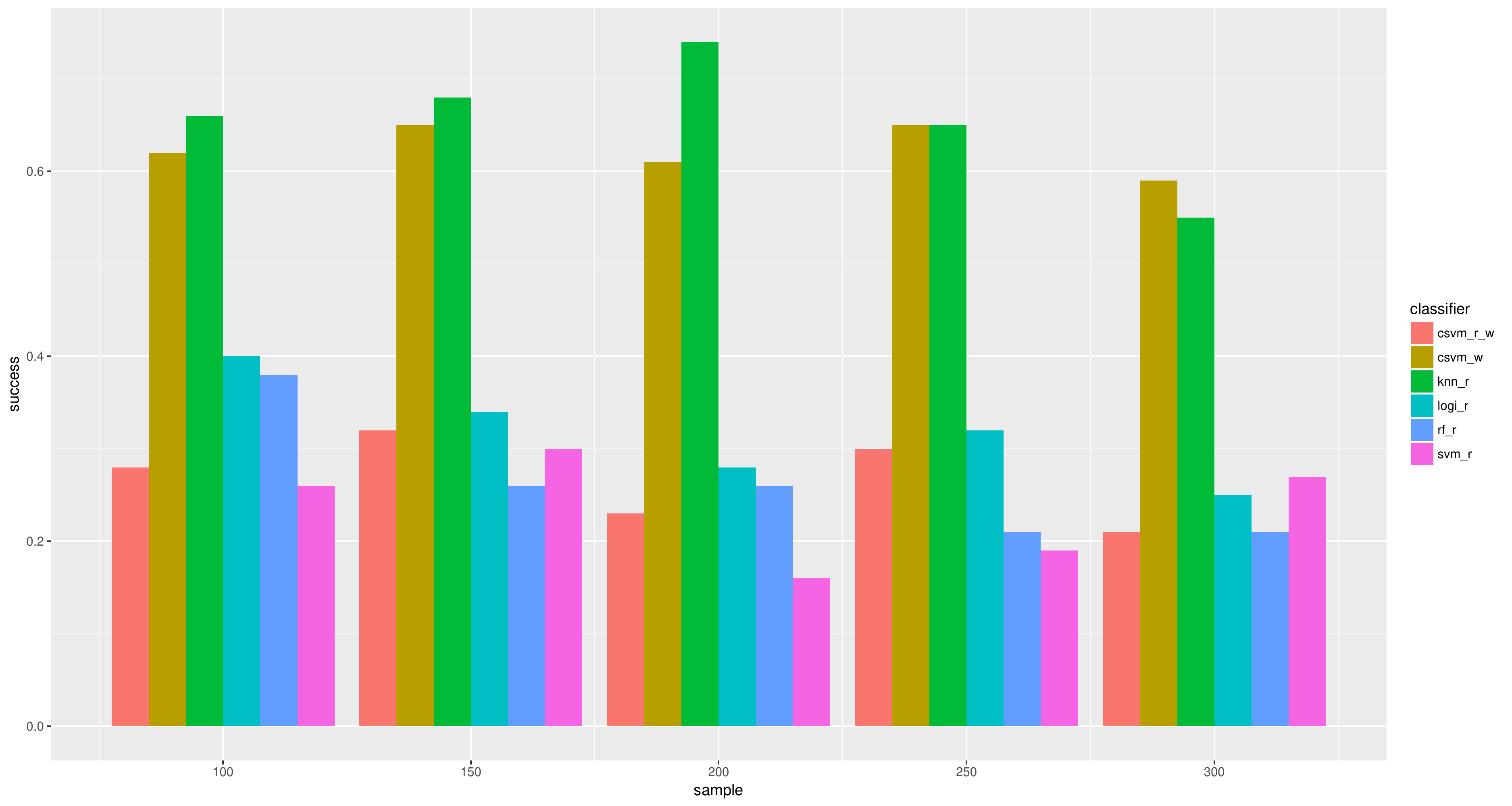}
	\caption{Success rates for all the models. We can see that weighted CSVM and kNN has a larger success rates when the other three have similar success rates. But within those two groups, the proposed model always has a smaller ambiguity.}
\end{figure}

\subsection*{High-dimensional donut}

\begin{figure}[H]
	\includegraphics[height = 6cm,width=1\textwidth]{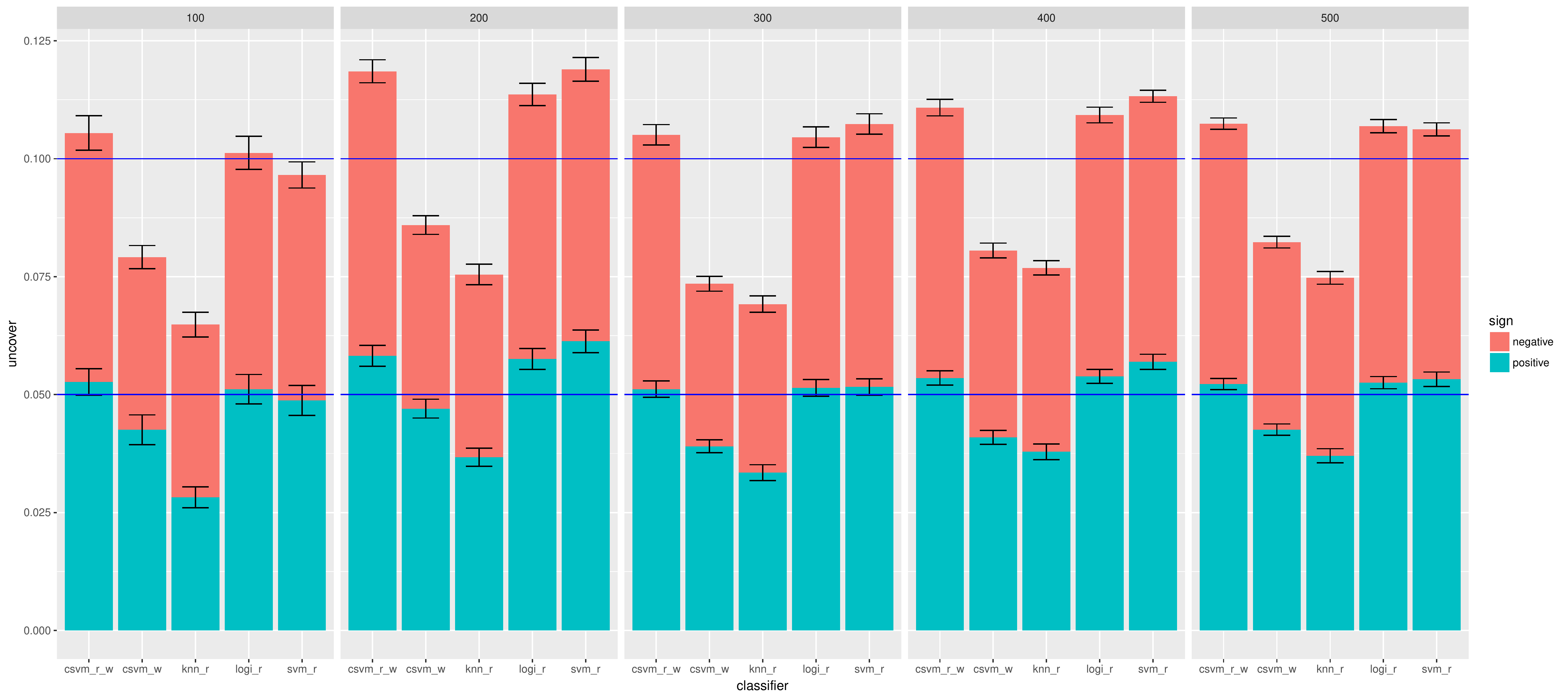}
	\caption{Non-coverage rates for all the models. We can see that weighted CSVM and kNN has a smaller non-coverage rates when the other three have similar non-coverage rates. But within those two groups, the proposed model always has a smaller ambiguity.}
\end{figure}
\begin{figure}[H]
	\includegraphics[height = 6cm,width=1\textwidth]{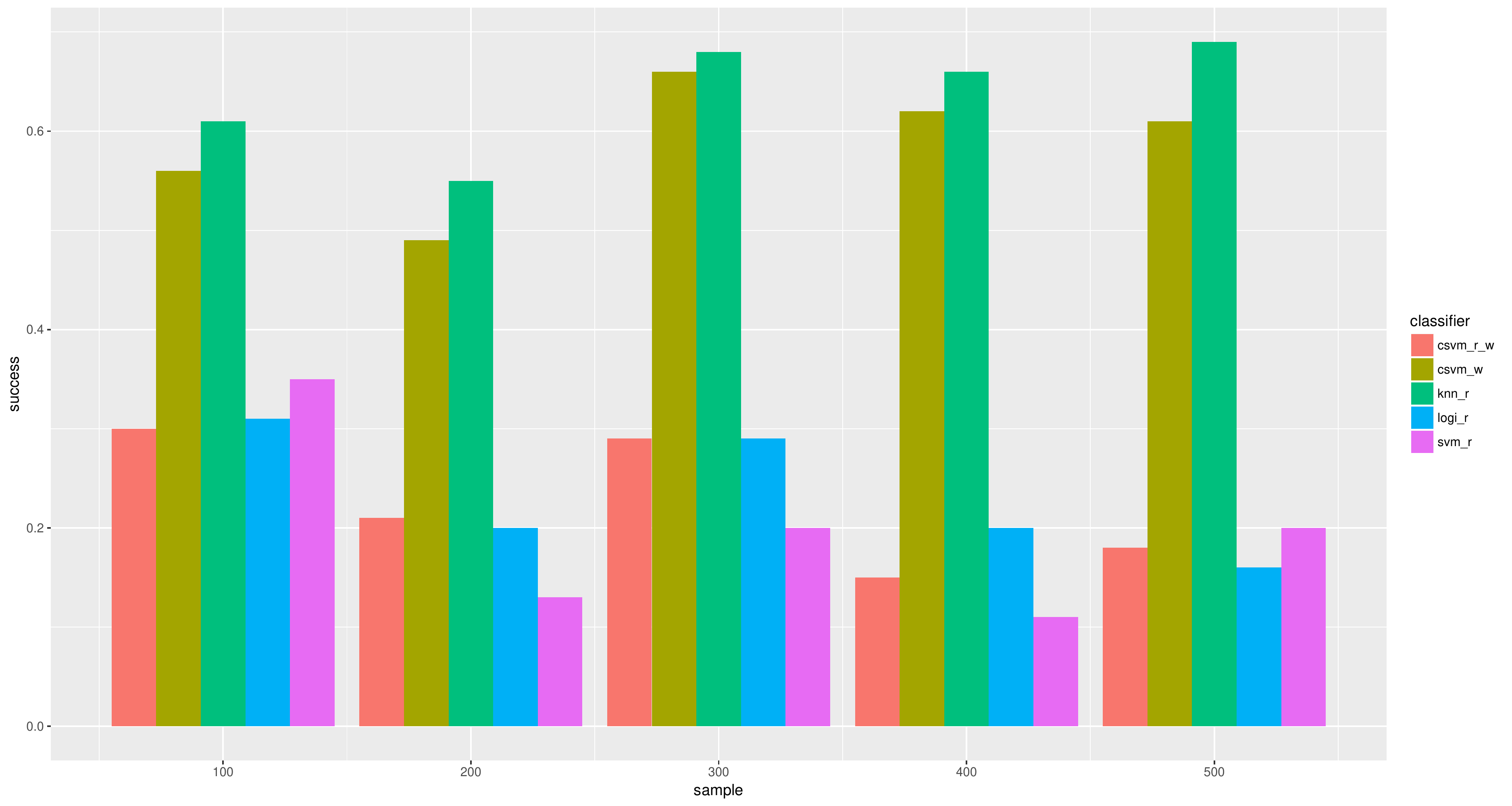}
	\caption{Success rates for all the models. We can see that weighted CSVM and kNN has a larger success rates when the other three have similar success rates. But within those two groups, the proposed model always has a smaller ambiguity..}
\end{figure}

\bibliographystyle{plainnat}
\bibliography{LCSuSVM}